\def\year{2019}\relax
\documentclass[letterpaper]{article} 
\usepackage{aaai19}  
\usepackage{times}  
\usepackage{helvet} 
\usepackage{courier}  
\usepackage[hyphens]{url}  
\usepackage{graphicx} 
\usepackage{enumerate}
\usepackage{amssymb}
\usepackage{amsmath}
\usepackage{amsthm}
\usepackage{bm}
\usepackage{subcaption}
\newcommand{\citet}[1]{\citeauthor{#1} \shortcite{#1}}
\newcommand{\citep}{\cite}

\newtheorem{theorem}{Theorem}
\newtheorem{lemma}{Lemma}

\newtheorem{corollary}{Corollary}
\newtheorem{definition}{Definition}
\newtheorem{assumption}{Assumption}

\newcommand{\bw}{\mathbf{w}}

\newcommand{\acal}{\mathcal{A}}
\newcommand{\vcal}{\mathcal{V}}
\newcommand{\fcal}{\mathcal{F}}
\newcommand{\gcal}{\mathcal{G}}

\newcommand{\zcal}{\mathcal{Z}}

\newcommand{\dcal}{\mathcal{D}}

\newcommand{\ncal}{\mathcal{N}}
\newcommand{\ebb}{\mathbb{E}}

\newcommand{\rbb}{\mathbb{R}}

\newcommand{\bTheta}{\bm{\Theta}}
\DeclareMathOperator{\argmax}{arg\,max}
\DeclareMathOperator{\argmin}{arg\,min}

\nocopyright
\title{On Performance Estimation in Automatic Algorithm Configuration}

\author{
Shengcai Liu,\textsuperscript{\rm 1}
Ke Tang,\textsuperscript{\rm 2}
Yunwen Lei,\textsuperscript{\rm 2}
Xin Yao\textsuperscript{\rm 2}
\\
\textsuperscript{\rm 1} School of Computer Science and Technology,
University of Science and Technology of China, Hefei 230027, China\\
\textsuperscript{\rm 2} Department
of
Computer Science and Engineering,
Southern University of Science and Technology, Shenzhen 518055, China\\
liuscyyf@mail.ustc.edu.cn,
\{tangk3, leiyw, xiny\}@sustech.edu.cn
}

\begin{document}
\maketitle
\begin{abstract}
Over the last decade, research on 
automated parameter tuning,
often referred to as automatic algorithm configuration (AAC),
has made significant progress.
Although the usefulness of such tools has been
widely recognized in real world applications,
the theoretical foundations of AAC are still very weak.
This paper addresses this gap by
studying the performance estimation problem in AAC.
More specifically, this paper
first proves the universal best performance estimator
in a practical setting,
and then establishes theoretical bounds on the estimation error,
i.e., the difference between the training
performance and the true performance
for a parameter configuration,
considering finite and infinite configuration spaces respectively.
These findings were verified in extensive experiments
conducted on four algorithm configuration scenarios involving
different problem domains.
Moreover, insights for 
enhancing existing AAC methods are also identified.
\end{abstract}

\section{Introduction}
Many high-performance algorithms
for solving computationally hard problems,
ranging from exact methods such as mixed integer
programming solvers to heuristic methods such
as local search, 
involve a large number of free parameters that need to be
carefully tuned to achieve their best performance.
In many cases, finding performance-optimizing
parameter settings is performed manually in an ad-hoc way.
However, the manually-tuning approach is often intensive
in terms of human effort \cite{lopez2016irace}
and thus there are a lot of attempts on automating this process
(see \cite{hutter2009paramils} for a comprehensive review),
which is usually referred to as
automatic algorithm configuration (AAC) \cite{Hoos121}.
Many AAC methods such as ParamILS \cite{hutter2009paramils},
GGA/GGA+\cite{ansotegui2009gender,AnsoteguiMSST15},
irace \cite{lopez2016irace} and SMAC \cite{hutter2011sequential}
have been proposed in the last few years.
They have been
used for boosting the algorithm's performance in
a wide range of domains such as
the boolean satisfiability problem (SAT) \cite{hutter2009paramils},
the traveling salesman problem (TSP) \cite{lopez2016irace,LiuT019},
the answer set programming (ASP) \cite{HutterLFLHLS14}
and machine learning \cite{FeurerKESBH15,KotthoffTHHL17}.

Despite the notable success achieved in application,
the theoretical aspects of AAC have been rarely investigated.
To our best knowledge, for AAC the first
theoretical analysis was given by \citet{Birattari2004},
in which the author analyzed expectations and
variances of different performance estimators that estimate the
true performance of a given parameter configuration on the basis
of $N$ runs of the configuration. It
is concluded in \cite{Birattari2004} that
performing one single run on $N$ different problem instances guarantees
that the variance of the estimate is minimized,
which has served as a guidance in the design of the
performance estimation mechanisms in later AAC methods including irace, ParamILS and SMAC. 
It is noted that the analysis in \cite{Birattari2004} assumes
that infinite problem instances could be sampled
for configuration evaluation.
However, in practice we are often only given
a set of finite training instances \cite{Hoos121}.

Recently, \citet{KleinbergLL17} introduced
a new algorithm configuration framework named
Structured Procrastination (SP),
which is guaranteed to find an approximately optimal parameter configuration
within a logarithmic factor of the optimal runtime in a worst-case sense.
Furthermore, the authors showed
that the gap between worst-case runtimes of
existing methods (ParamILS, GGA, irace, SMAC) and SP
could be arbitrarily large.
These results were later extended in
\cite{WeiszGS18,WeiszGS19}, in which the authors
proposed new methods, called LEAPSANDBOUNDS (LB) and
CapsAndRuns (CR),
with better runtime guarantees.
However, there is a discrepancy between
the algorithm configuration problem
addressed by these methods (SP, LB and CR)
and the problem that is most
frequently encountered in practice.
More specifically, these methods are designed to
find parameter configurations with approximately optimal 
performances on the input (training) instances;
while in practice it is more desirable
to find parameter
configurations that will perform well on
new unseen instances rather than
just the training instances \cite{Hoos121}.
Indeed, one of the most critical issues 
that needs to be addressed in AAC is the over-tuning
phenomenon \cite{Birattari2004},
in which the found parameter configuration is with 
excellent training performance,
but performs badly on new instances
\footnote{
To appropriately evaluate AAC methods,
in the literature,
including widely used benchmarks (e.g., AClib \cite{HutterLFLHLS14})
and major contests (e.g., the Configurable SAT Solver Challenge (CSSC) \cite{HutterLBBHL17}),
the common scheme is to use an independent test set
that has never been used during the configuration procedures to
test the found configurations.}.

Based on the above observation, this paper extends
the results of \cite{Birattari2004} in several aspects.
First, this paper introduces a new formulation of the algorithm configuration problem
(Definition~\ref{def:AAC_definition}),
which concerns the optimization of the
expected performance of the configured 
algorithm on an instance distribution $\dcal$.
Compared to the one considered by \citet{Birattari2004}
in which $\dcal$ is directly given
(thus could be sampled infinitely),
in the problem considered here $\dcal$
is unknown and inaccessible,
and the assumption
is that the input training instances (and the test instances) are
sampled $i.i.d$ from $\dcal$.
Therefore when solving this configuration problem,
we can only use the given finite training instances.
One key difficulty is that
the true performance of a parameter configuration is unachievable.
Subsequently, we could only run a configuration on
the training instances to obtain an estimate
of its true performance.
Thus a natural and important question is that,
given finite computational budgets,
e.g., $N$ runs of the configuration,
how to allocate them over the training instances
to obtain the most reliable estimate.
Moreover, given that we could obtain an
estimate of the true performance,
is it possible to quantify the difference
between the estimate and the true performance?

The second and the most important contribution of this paper
is that it answers the above questions theoretically.
More specifically,
this paper first introduces a universal best
performance estimator (Theorem~\ref{theorem:best_estimator})
that always distributes
the $N$ runs of a configuration to all training instances
as evenly as possible, such that
the performance estimate is most reliable. 
Then this paper investigates the estimation error, 
i.e., the difference between the training performance (the estimate)
and the true performance, and establishes
a bound on the estimation error
that holds for all configurations in the
configuration space, considering the cardinality
of the configuration space is finite (Theorem~\ref{theorem:finite}).
It is shown that the bound
deteriorates as the number of the considered configurations increases.
Since in practice the cardinality
of the configuration space
considered could be considerably large or even infinite,
by making two mild assumptions
on the considered configuration scenarios,
we remove the dependence on  the cardinality
of the configuration space
and finally establish a new bound on the estimation error
(Theorem~\ref{theorem:infinite}).

The effectiveness of these results have been 
verified in extensive experiments conducted on four configuration
scenarios involving problem domains including SAT, ASP and TSP.
Some potential directions for improving current
AAC methods from these results
have also been identified.

\section{Algorithm Configuration Problem}
\label{sec:2}
In a nutshell, the algorithm configuration problem concerns optimization
of the free parameters of a given parameterized algorithm (called target algorithm)
for which the performance is optimized.

Let $\acal$ denote the target algorithm and
let $p_{1},...,p_{h}$ be parameters of $\acal$.
Denote the set of possible values for each parameter $p_{i}$
as $\Theta_{i}$.
A parameter configuration $\theta$ (or simply configuration) 
of $\acal$ refers to a complete setting of $p_{1},...,p_{h}$,
such that the behavior of $\acal$
on a given problem instance is completely specified (up to
randomization of $\acal$ itself).
The configuration space $\bTheta = \Theta_{1} \times \Theta_{2}...\times \Theta_{h}$
contains all possible configurations of $\acal$.
For brevity, henceforth we will not distinguish
between $\theta$ and the instantiation of $\acal$
with $\theta$.
In real application $\acal$ is often randomized and its output is determined
by the used configuration $\theta$, the input instance $z$
and the random seed $v$.
Let $\dcal$ denote a probability distribution over a
space $\zcal$ of problem instances from which
$z$ is sampled.
Let $\gcal$ be
a probability distribution 
over a space $\vcal$ of random seeds from which $v$ is sampled.
In practice $\gcal$ is often given implicitly through a random
number generator.

Given an instance $z$ and a seed $v$, the quality of $\theta$ at $(z, v)$
is measured by a utility function $f_{\theta}:\zcal \times \vcal \rightarrow [L,U]$,
where $L,U$ are bounded real numbers.
In practice, it means running $\theta$ with $v$ on $z$,
and maps the result of this run to a scalar score.
Note how the mapping is done depends on the considered performance metric.
For examples, if we are interested in optimizing quality of
the solutions found by $\acal$,
then we might take the (normalized) cost of the solution
output by $\acal$ as the utility;
if we are interested in minimizing computational
resources consumed by $\acal$ (such as runtime,
memory or communication bandwidth),
then we might take the quantity of the consumed
resource of the run as the utility.
No matter which performance metric is considered,
in practice the value of $f_{\theta}$ is bounded
for all $\theta \in \bTheta$,
i.e., for all $\theta \in \bTheta$ and all
$(z,v) \in \zcal \times \vcal$,
$f_{\theta}(z,v) \in [L, U]$.

To measure the performance of $\theta$,
the expected value of the utility scores of $\theta$ across
different $(z,v)$, which is the most widely adopted criterion
in AAC applications
\cite{Hoos121}, is considered here.
More specifically,
as presented in Definition~\ref{def:AAC_definition},
the performance of $\theta$, denoted as $u(\theta)$, is its expected utility
score over instance distribution $\dcal$ and
random seed distribution $\gcal$. 
Without loss of generality, we always assume a smaller value is better
for $u(\theta)$.
The goal of the algorithm configuration problem is to find a
configuration from the configuration
space $\bTheta$ with the best performance.

\begin{definition}[Algorithm Configuration Problem]
\label{def:AAC_definition}
  Given a target algorithm $\acal$ with configuration space $\bTheta$,
  an instance distribution $\dcal$ defined over space $\zcal$,
  a random seed distribution $\gcal$ defined over space $\vcal$ and
  a utility function $f_{\theta}:\zcal \times \vcal \rightarrow [L,U]$ that
  measures the quality of $\theta$ at $(z, v)$,
  the algorithm configuration problem is to find a configuration $\theta^{\star}$ from
  $\bTheta$ with the best performance:
  \[
  \label{eq:denitionofutheta}
    \theta^{\star} \in \argmin_{\theta \in \bTheta} u(\theta),
  \]
  where  $u(\theta) = \ebb_{z \sim \dcal, v \sim \gcal}[f_{\theta}(z,v)]$.
\end{definition}
In practice, $\dcal$ is usually unknown
and the analytical
solution of $u(\theta)$ is unachievable.
Instead, usually we have a set of problem instances $\{z_{1},...,z_{K}\}$, called training instances, which are assumed
to be sampled $i.i.d$ from $\dcal$.
To estimate $u(\theta)$, a series of experiments of $\theta$ on $\{z_1, z_2,...,z_K\}$
could be run.
As presented in Definition~\ref{def:exsetting},
an experimental setting $S_{N}$ to estimate $u(\theta)$
is to run $\theta$ on $\{z_{1},...,z_{K}\}$ for $N$ times,
each time with a random seed sampled $i.i.d$ from $\gcal$.
\begin{definition}[Experimental Setting $S_{N}$]
\label{def:exsetting}
  Given a configuration $\theta$, a set of $K$ training instances $\{z_{1},...,z_{K}\}$ and 
  the total number $N$ of runs of $\theta$,
  an experimental setting $S_{N}$ to estimate $u(\theta)$
  is a list of $N$ tuples, in which each tuple $(z,v)$ consists
  of an instance $z$ and a random seed $v$, meaning a single run
  of $\theta$ with $v$ on $z$. Let $n_{i}$ denote the number of runs
  performed on $z_{i}$ (note $n_{i}$ could be 0, meaning $\theta$
  will not be run on $z_{i}$).
  It holds that $\sum_{i=1}^{K} n_{i}=N$
  and $S_{N}$ could be written as:
  \[
    \begin{aligned}
    S_{N}=[&(z_{1}, v_{1,1}),...,(z_{1},v_{1,n_{1}}),...,(z_{i}, v_{i,1}),...,\\
    &(z_{i},v_{i,n_{i}}),...,(z_{K}, v_{K,1}),...,(z_{K},v_{K,n_{K}})].
    \end{aligned}
  \]
\end{definition}
After performing the $N$ runs of $\theta$ as specified in $S_{N}$,
the utility scores of these runs are aggregated to 
estimate $u(\theta)$.
The following estimator $\hat{u}_{S_{N}}(\theta)$, which calculates the
mean utility across all runs and 
is widely adopted in AAC methods \cite{hutter2009paramils,lopez2016irace,hutter2011sequential},
is presented in Definition~\ref{def:estimator}.

\begin{definition}[Estimator $\hat{u}_{S_{N}}(\theta)$]
\label{def:estimator}
  Given a configuration $\theta$ and an experimental setting $S_{N}$,
  the training performance of $\theta$, which is 
  an estimate of $u(\theta)$, is given by:
  \[
  \label{eq:estimator}
    \hat{u}_{S_{N}}(\theta) = \frac{1}{N}\sum_{i=1}^{K}\sum_{j=1}^{n_{i}}f_{\theta}(z_{i}, v_{i, j}).
  \]
\end{definition}

Since different experimental settings represent different performance estimators,
which have different behaviors.
It is thus necessary to investigate which $S_N$ 
is the best.

\section{Universal Best Performance Estimator}
\label{section3}
To determine the values of $n_{1},...,n_{K}$ in $S_{N}$,
\citet{Birattari2004} analyzed expectations and variances of
$\hat{u}_{S_{N}}(\theta)$,
and concluded that $\hat{u}_{S^{\circ}_{N}}(\theta)$
with $S^{\circ}_{N}:=[(z_{1}, v_{1,1}),(z_{2},v_{2,1}),...,(z_{N},v_{N,1})]$
has the minimal variance.
It is noted that the analysis in \cite{Birattari2004}
assumes that infinite problem instances could be sampled from $\dcal$;
thus for performing $N$ runs of $\theta$, as specified in $S^{\circ}_{N}$,
it is always the best to
sample $N$ instances from $\dcal$ and perform one single run of $\theta$ on each
instance.
In other words, $S^{\circ}_{N}$ is established
on the basis that the number of the training
instances $K$ could always be set equal to $N$.
However, in practice usually we only have a finite number of training instances.
In the case that $K \not= N$, which $S_{N}$ is the 
best?
Theorem~\ref{theorem:best_estimator} answers this question for arbitrary relationship
between $K$ and $N$.
Before presenting Theorem~\ref{theorem:best_estimator},
some necessary definitions are introduced.

Given a configuration $\theta$ and an instance $z$,
the expected utility of $\theta$ within $z$,
denoted as $u_{z}(\theta)$, is $\ebb_{\gcal}[f_{\theta}(z,v)|z]$.
The variance of the utility of $\theta$ within $z$,
denoted as $\sigma^{2}_{z}(\theta)$,
is $\ebb_{\gcal}[(f_{\theta}(z,v)-u_{z}(\theta))^{2}|z]$.
Based on $u_{z}(\theta)$ and $\sigma^{2}_{z}(\theta)$,
the expected within-instance variance $\bar{\sigma}^{2}_{WI}(\theta)$ of $\theta$
and the across-instance variance $\bar{\sigma}^{2}_{AI}(\theta)$ of $\theta$
are defined in Definition~\ref{def:within_instance} and
Definition~\ref{def:across_instance}, respectively.
\begin{definition}[Expected within-instance Variance of $\theta$]
  $\bar{\sigma}^{2}_{WI}(\theta)$
  is the expected value of $\sigma^{2}_{z}(\theta)$ over instance distribution $\dcal$:
  \small
  \[
    \bar{\sigma}^{2}_{WI}(\theta) = \ebb_{\dcal}[\sigma^{2}_{z}(\theta)].
  \]
\label{def:within_instance}  
\end{definition}
\begin{definition}[Across-instance Variance of $\theta$]
  $\bar{\sigma}^{2}_{AI}(\theta)$
  is the variance of $u_{z}(\theta)$ over instance distribution $\dcal$:
  \small
  \[
    \bar{\sigma}^{2}_{AI}(\theta) = \ebb_{\dcal}[(u_{z}(\theta) - u(\theta))^{2}].
  \]
\label{def:across_instance}
\end{definition}
The expectation and the variance of an estimator $\hat{u}_{S_{N}}(\theta)$ are presented
in Lemma~\ref{lem:unbiased} and Lemma~\ref{lem:variance}, respectively.
The proofs are omitted here due to space limitations.

\begin{lemma}
\label{lem:unbiased}
 The expectation of $\hat{u}_{S_{N}}(\theta)$ is $u(\theta)$, that is,
 $\hat{u}_{S_{N}}(\theta)$ is an unbiased estimator of $u(\theta)$ no matter
 how $n_{1},...,n_{K}$ in $S_{N}$ are set:
 \small
 \[
 	 \ebb_{S_{N}}[\hat{u}_{S_{N}}(\theta)] = u(\theta).
 \] 
\end{lemma}

\begin{lemma}
\label{lem:variance}
 The variance of $\hat{u}_{S_{N}}(\theta)$ is given by:
 \small
 \begin{equation}
 \label{eq:variance}
  	 \ebb_{S_{N}}[(\hat{u}_{S_{N}}(\theta)-u(\theta))^{2}] = \frac{1}{N} \bar{\sigma}^{2}_{WI}(\theta) + \frac{\Sigma_{i=1}^{K}n_{i}^{2}}{N^{2}} \bar{\sigma}^{2}_{AI}(\theta).
 \end{equation}
\end{lemma}

\begin{theorem}
\label{theorem:best_estimator}
  Given a configuration $\theta$, a training set of $K$ instances and the total number $N$
  runs of $\theta$, the universal best estimator $\hat{u}_{S^{\ast}_{N}}(\theta)$
  for $u(\theta)$ is obtained by setting
  $S^{\ast}_{N}:= n_{i} \in \{\lfloor{\frac{N}{K}}\rfloor, \lceil{\frac{N}{K}}\rceil \} $
  for all $i \in \{1,2,...,K\}$, s.t. $\sum_{i=1}^{K} n_{i}=N$.
  $\hat{u}_{S^{\ast}_{N}}(\theta)$ is an unbiased estimator of $u(\theta)$ and
  is with the minimal variance among all possible estimators.
\end{theorem}
\begin{proof}

By Lemma~\ref{lem:unbiased}, $\hat{u}_{S^{\ast}_{N}}(\theta)$ is an unbiased estimator of $u(\theta)$.
  We now prove $\hat{u}_{S^{\ast}_{N}}(\theta)$ has the minimal variance. By Lemma~\ref{lem:variance}, the variance of $\hat{u}_{S_{N}}(\theta)$ is $\frac{1}{N} \bar{\sigma}^{2}_{WI}(\theta) + \frac{\Sigma_{i=1}^{K}n_{i}^{2}}{N^{2}} \bar{\sigma}^{2}_{AI}(\theta)$.
  Since $N$ and $K$ are fixed, and $\bar{\sigma}^{2}_{WI}(\theta)$ and $\bar{\sigma}^{2}_{AI}(\theta)$ are constants for a given $\theta$, we need to minimize
  $\sum_{i=1}^{K}n_{i}^{2}$, s.t. $\sum_{i=1}^{K}n_{i}=N$.
  Define $Q_{n}=\sqrt{\sum_{k=1}^{K}n_{i}^{2}}$ and $\bar{n}=\frac{\sum_{i=1}^{K}n_{i}}{K}=\frac{N}{K}$,
  it then follows that
  $Q_{n}^{2}=K\bar{n}^{2}+\sum_{i=1}^{K}(n_{i}-\bar{n}^{2})$.
  Then it suffices to prove that $Q_{n}^{2}$ is minimized on the condition $n_{i} \in \{\lfloor{\frac{N}{K}}\rfloor, \lceil{\frac{N}{K}}\rceil \}$ for all $i \in \{1,2,...,K\}$.
  Assuming $Q_{n}^{2}$ is minimized while the condition not satisfied, then there must exist
  $n_{i}$ and $n_{j}$, such that $n_{i}-n_{j}>1$;
  then we have
  $(n_{i}-\bar{n})^{2}+(n_{j}-\bar{n})^{2} > (n_{i}-\bar{n})^{2}+(n_{j}-\bar{n})^{2}-2(n_{i}-n_{j})+2 = (n_{i}-\bar{n}-1)^{2}+(n_{j}-\bar{n}+1)^{2}$.
  This contradicts the assumption that $Q_{n}^{2}$ is minimized.
The proof is complete.
\end{proof}

Theorem~\ref{theorem:best_estimator} states that it is always the best
to distribute the $N$ runs of $\theta$ to all training instances
as evenly as possible, in which case $\max_{i,j \in \{1,...,K\}} |n_{i}-n_{j}| \leq 1$,
no matter $K \neq N$ or $K=N$.
When $K=N$, $S^{\ast}_{N}(\theta)$ is actually equivalent to  $S^{\circ}_{N}$
that performs one single run of $\theta$ on each instance.
When $K \neq N$, $S^{\ast}_{N}(\theta)$ will
perform $\lceil\frac{N}{K}\rceil$ runs of $\theta$ on each of $(N\ \mathrm{mod}\ K)$ instances
and perform $\lfloor\frac{N}{K}\rfloor$ runs on each of the rest instances.
It is worth mentoring that
practical AAC methods including ParamILS,
SMAC and irace actually
adopt the same or quite similar estimators as
$S^{\ast}_{N}(\theta)$.
Theorem~\ref{theorem:best_estimator} provides a theoretical guarantee
for these estimators, and
$S^{\ast}_{N}(\theta)$ will be further evaluated in the experiments.

\section{Bounds on Estimation Error}
Although Theorem~\ref{theorem:best_estimator} presents the estimator with the
universal minimal variance, it cannot provide any information about
how large the estimation error, i.e., $u(\theta)-\hat{u}_{S_{N}}(\theta)$, could be.
Bounds on estimation error are useful in both theory and practice because 
we could use them to establish bounds on the true performance $u(\theta)$,
given that in algorithm configuration process the training performance
$\hat{u}_{S_{N}}(\theta)$ is actually known.
In general,
given a configuration $\theta$,
its training performance $\hat{u}_{S_{N}}(\theta)$
is a random variable
because the training instances and the random seeds specified in $S_{N}$
are drawn from distributions $\dcal$ and $\gcal$, respectively.
Thus we focus on establishing probabilistic
inequalities for $u(\theta)-\hat{u}_{S_{N}}(\theta)$,
i.e., for any $0<\delta<1$, with probability at least $1-\delta$,
there holds $u(\theta)-\hat{u}_{S_{N}}(\theta) \leq A(\delta)$.
In particular, probabilistic bounds on uniform estimation error,
i.e., $\sup_{\theta \in \bTheta}[u(\theta)-\hat{u}_{S_{N}}(\theta)]$,
that hold for all $\theta \in \bTheta$ are established.
Recalling that Lemma~\ref{lem:unbiased} states 
$\ebb_{S_{N}}[\hat{u}_{S_{N}}(\theta)] = u(\theta)$,
the key technique for deriving bounds on 
$u(\theta)-\hat{u}_{S_{N}}(\theta)$ is the concentration
inequality presented in Lemma~\ref{lem:concen} that bounds how $\hat{u}_{S_{N}}(\theta)$
deviates from its expected value $u(\theta)$.

\begin{lemma}[Bernstein's Inequality \cite{bernstein1927theory}]
 Let $X_{1},X_{2},...,X_{n}$ be independent centered bounded random variables, i.e., $\mathrm{Prob}\{|X_{i}| \leq a\}=1$ and $\ebb[X_{i}]=0$. 
 Let $\sigma^{2}=\frac{1}{n}\sum_{i=1}^{n}Var[X_{i}]$ where $Var[X_{i}]$ is the variance of $X_{i}$.
 Then for any $\epsilon>0$ we have
 \[
   \mathrm{Prob}\{\frac{1}{n}\sum_{i=1}^{n}X_{i} \geq \epsilon\} \leq \mathrm{exp}(-\frac{n\epsilon^{2}}{2\sigma^{2}+\frac{2a\epsilon}{3}}).
 \]
\label{lem:berstein}
\end{lemma}

\begin{lemma}
  Given a configuration $\theta$, an experimental setting
  $S_{N}=[(z_{1}, v_{1,1}),...,(z_{K},v_{K,n_{K}})]$
  and a performance estimator $\hat{u}_{S_{N}}(\theta)=\frac{1}{N}\sum_{i=1}^{K}\sum_{j=1}^{n_{i}}f_{\theta}(z_{i}, v_{i, j})$.
  Let $\tau_{\theta}^{2}=\bar{\sigma}^{2}_{WI}(\theta) + \frac{\sum_{i=1}^{K}n_{i}^{2}}{N}\bar{\sigma}^{2}_{AI}(\theta)$.
  Let $C=U-L$, where $L,U$ are the lower bound and the upper
  bound of $f_{\theta}$ respectively (see Definition~\ref{def:AAC_definition}),
  and let $n=\max\{n_{1},n_{2},...,n_{K}\}$.
  Then for any $\epsilon>0$, we have
  \[
    \mathrm{Prob}\{u(\theta) - \hat{u}_{S_{N}}(\theta)
    \geq \epsilon \}
    \leq \mathrm{exp}(-\frac{N\epsilon^{2}}{2\tau_{\theta}^{2}+\frac{2nC\epsilon}{3}}).
  \]
\label{lem:concen}
\end{lemma}
\begin{proof}
 Define random variables $x_{i,j}=u(\theta)-f_{\theta}(z_{i}, v_{i,j})$,
 and define random variables $X_{i}=\sum_{j=1}^{n_{i}}x_{i,j}$.
 First we prove that $X_{1},...,X_{K}$ satisfy the conditions in Lemma~\ref{lem:berstein}.
 $\ebb[X_{i}]=\sum_{j=1}^{n_{i}}\ebb[x_{i,j}]=\sum_{j=1}^{n_{i}}[u(\theta)-\ebb[f_{\theta}(z_{i},v_{i,j})]]=0$.
 By Definition~\ref{def:AAC_definition},
 $\mathrm{Prob}\{L \leq f_{\theta}(z_{i},v_{i,j}) \leq U\}=1$,
 it holds that $L \leq u(\theta) \leq U$ (since $u(\theta) = \ebb[f_{\theta}(z_{i},v_{i,j})]$).
 Thus we have $\mathrm{Prob}\{|x_{i,j}| \leq U-L\}=1$
 and
 $\mathrm{Prob}\{|X_{i}| \leq n(U-L) \}=1$.
 For any $p \neq q$,
 $X_{p}$ and $X_{q}$ are independent.
 Thus $X_{1},X_{2},...,X_{K}$ are independent random variables.

 Let $\bar{X}=\frac{1}{K}\sum_{i=1}^{K}X_{i}$.
 By Lemma~\ref{lem:berstein}, it holds that, for any $\epsilon>0$,
 $\mathrm{Prob}\{\bar{X} > \epsilon\} \leq \mathrm{exp}(-\frac{K\epsilon^{2}}{2\sigma^{2}+\frac{2C\epsilon}{3}})$,
 where $\sigma^{2}=\frac{1}{K}\sum_{i=1}^{K}Var[X_{i}]$.
 Notice that $\frac{K}{N}\bar{X}=u(\theta) - \hat{u}_{S_{N}}$;
 thus it holds that for any $\epsilon>0$,
 \begin{equation}
 \label{eq:midresult1}
  \mathrm{Prob}\{u(\theta) - \hat{u}_{S_{N}}>\epsilon\} \leq \mathrm{exp}(-\frac{N\epsilon^{2}}{\frac{2K}{N}\sigma^{2}+\frac{2nC\epsilon}{3}}).
 \end{equation} 
  The rest of the proof focuses on $\sigma^{2}$.
  Since $E[x_{i}]=0$,
  $Var[X_{i}]=\ebb[(X_{i}-E[X_{i}])^{2}]=\ebb[X_{i}^{2}]$.
  Substitute $X_{i}$ with $\sum_{j=1}^{n_{i}}x_{i,j}$ and 
  we have
  $Var[X_{i}]=\sum_{j=1}^{n_{i}}\ebb[x_{i,j}^{2}]
  +\sum_{1 \leq j < l \leq n_{i}}2\ebb[x_{i,j}x_{i,l}]$.
  We analyze $\ebb[x_{i,j}^{2}]$ and $\ebb[x_{i,j}x_{i,l}]$ in turn.
  $\ebb[x_{i,j}^{2}]=Var[x_{i,j}]+\ebb[x_{i,j}]^{2}=Var[x_{i,j}]+0=\bar{\sigma}^{2}_{WI}(\theta)+\bar{\sigma}^{2}_{AI}(\theta)$ (by setting $N,K=1$ in Eq.~(\ref{eq:variance})).
  $\ebb[x_{i,j}x_{i,l}]=\ebb[(f_{\theta}(z_{i},v_{i,j})-u(\theta))(f_{\theta}(z_{i},v_{i,l})-u(\theta))]=\ebb[(f_{\theta}(z_{i},v_{i,j})(f_{\theta}(z_{i},v_{i,l})]-u(\theta)^{2}$.
  Given an instance $z_{i}$, $f_{\theta}(z_{i},v_{i,j})$ and $f_{\theta}(z_{i},v_{i,l})$ are independent because $v_{i,j}$ and $v_{i,l}$ are sampled $i.i.d$ from $\gcal$.
Thus it holds that:
\small
\begin{align}
&\ebb[(f_{\theta}(z_{i},v_{i,j})(f_{\theta}(z_{i},v_{i,l})]
= \ebb_{\dcal}[\ebb_{\gcal} [f_{\theta}(z_{i},v_{i,j}) f_{\theta}(z_{i},v_{i,l}) |z_i]] \nonumber \\
&= \ebb_{\dcal}[\ebb_{\gcal} [f_{\theta}(z_{i},v_{i,j})|z_{i}]
\ebb_{\gcal} [f_{\theta}(z_{i},v_{i,l})|z_{i}]]
=\ebb_{\dcal}[u_{z_{i}}(\theta)^{2}]. \nonumber
\end{align}
By the fact
$\ebb_{\dcal}[u_{z}(\theta)]=u(\theta)$,
$\ebb_{\dcal}[u_{z_i}(\theta)^{2}] - u(\theta)^{2}
=\ebb_{\dcal}[(u_{z_i}(\theta)-u(\theta))^{2}]=\bar{\sigma}^{2}_{AI}(\theta)$.
The last step is by Definition~\ref{def:across_instance}.
 
Summing up the above results, we have
  $Var[X_{i}]=n_{i}(\bar{\sigma}^{2}_{WI}(\theta)+\bar{\sigma}^{2}_{AI}(\theta))+n_{i}(n_{i}-1)\bar{\sigma}^{2}_{AI}(\theta)
  =n_{i}\bar{\sigma}^{2}_{WI}(\theta) + n_{i}^{2}\bar{\sigma}^{2}_{AI}(\theta)$.
  Thus $\sigma^{2}=\frac{1}{K}\sum_{i=1}^{K}Var[X_{i}]=\frac{N}{K}\bar{\sigma}^{2}_{WI}(\theta) + \frac{\sum_{i=1}^{K}n_{i}^{2}}{K}\bar{\sigma}^{2}_{AI}(\theta)$.
  Substitute $\sigma^{2}$ in Eq.~(\ref{eq:midresult1}) with this result and the proof is complete.
\end{proof}

\subsection{On Configuration Space with Finite Cardinality}
Theorem~\ref{theorem:finite} presents the
bound for uniform estimation error
when $\bTheta$ is of finite cardinality.

\begin{theorem}
  \label{theorem:finite}
   Given a performance estimator $\hat{u}_{S_{N}}(\theta)$.
   Let $\theta^{\dagger}=\argmax_{\theta \in \bTheta}\tau_{\theta}^{2}$,
   where $\tau_{\theta}^{2} = \bar{\sigma}^{2}_{WI}(\theta) + \frac{\sum_{i=1}^{K}n_{i}^{2}}{N}\bar{\sigma}^{2}_{AI}(\theta)$,
   and let
   $\tau^{2}=\tau_{\theta^{\dagger}}^{2}$,
   $\bar{\sigma}^{2}_{WI}=\bar{\sigma}^{2}_{WI}(\theta^{\dagger})$
   and 
   $\bar{\sigma}^{2}_{AI}=\bar{\sigma}^{2}_{AI}(\theta^{\dagger})$.
   Let $n=\max\{n_{1},n_{2},...,n_{K}\}$
   and $C=U-L$.
   Given that $\bTheta$ is of finite cardinality, i.e.,
   $\bTheta=\{\theta_{1},\theta_{2},...,\theta_{m}\}$,
   then for any $0<\delta<1$, with probability at least $1-\delta$,
   there holds:
   \small
   \begin{align}
   \label{eq:finite}
    \sup_{\theta \in \bTheta}&[u(\theta)-\hat{u}_{S_{N}}(\theta)] \nonumber \\
    &\leq
    \frac{2nC\ln{\frac{m}{\delta}}}{3N}
    +
    \sqrt{2\ln{\frac{m}{\delta}}(\frac{1}{N}\bar{\sigma}^{2}_{WI} + \frac{\sum_{i=1}^{K}n^{2}_{i}}{N^{2}}\bar{\sigma}^{2}_{AI})}.
   \end{align}   
\end{theorem}
\begin{proof}
  By Lemma~\ref{lem:concen}, for a 
  given configuration $\theta$,
  for any $\epsilon>0$, it holds that
  $
    \mathrm{Prob}\{u(\theta) - \hat{u}_{S_{N}}(\theta)
    \geq \epsilon \}
    \leq \mathrm{exp}(-\frac{N\epsilon^{2}}{2\tau_{\theta}^{2}+\frac{2nC\epsilon}{3}})
  $.
  By union bound,
  $
  \mathrm{Prob}\{\sup_{\theta \in \bTheta}[u(\theta) - \hat{u}_{S_{N}}(\theta)]
    \geq \epsilon \}
    \leq \sum_{i=1}^{m}
    \mathrm{Prob}\{u(\theta_{i}) - \hat{u}_{S_{N}}(\theta_{i})
    \geq \epsilon\}
    \leq
    m \mathrm{exp}(-\frac{N\epsilon^{2}}{2\tau^{2}+\frac{2nC\epsilon}{3}}).
  $
  Let $\delta=m \mathrm{exp}(-\frac{N\epsilon^{2}}{2\tau^{2}+\frac{2nC\epsilon}{3}})    
  $,
  and $\epsilon$ is solved as:
  $\epsilon=\frac{1}{2N}
  [\frac{2nC}{3}\ln{\frac{m}{\delta}}
  + \sqrt{(\frac{2nC}{3}\ln{\frac{1}{\delta})^{2}
  +8N\tau^{2}\ln{\frac{m}{\delta}}
  }}]
  \leq
  \frac{2nC\ln{\frac{m}{\delta}}}{3N}
	+
	\sqrt{2\ln{\frac{m}{\delta}}(\frac{\tau^{2}}{N})}.
   $
   Substituting $\tau^{2}$ with
   $\bar{\sigma}^{2}_{WI} + \frac{\sum_{i=1}^{K}n_{i}^{2}}{N}\bar{\sigma}^{2}_{AI}$
   proves Theorem~\ref{theorem:finite}.
\end{proof}

Note that for different $S_{N}$,
the bounds on the right side of
Eq.~(\ref{eq:finite}) are different.
The proof of Theorem~\ref{theorem:best_estimator}
shows that $\sum_{i=1}^{K}n_{i}^{2}$, s.t. $\sum_{i=1}^{K}n_{i}=N$,
is minimized on the condition $n_{i} \in \{\lfloor{\frac{N}{K}}\rfloor, \lceil{\frac{N}{K}}\rceil \}$ for all $i \in \{1,2,...,K\}$.
Moreover, it is easy to verify that
$n=\max\{n_{1},n_{2},...,n_{K}\}$
is also minimized on the same condition,
in which case $n=\lceil{\frac{N}{K}\rceil}$.
Thus we can immediately obtain 
Corollary~\ref{cor:finite}.

\begin{corollary}
  \label{cor:finite}
   The estimator
  $\hat{u}_{S^{\ast}_{N}}$ established
  in Theorem~\ref{theorem:best_estimator},
  has the best bound for uniform estimation 
  error in Theorem~\ref{theorem:finite}.
  Given that $K$ divides $N$, 
  for any $0<\delta<1$, with probability at least $1-\delta$, there holds:
  \small
   \begin{equation*}
    \sup_{\theta \in \bTheta}[u(\theta)-\hat{u}_{S_{N}^{\ast}}(\theta)] \leq
    \frac{2C\ln{\frac{m}{\delta}}}{3K}
    +
    \sqrt{2\ln{\frac{m}{\delta}}(\frac{1}{N}\bar{\sigma}^{2}_{WI} + \frac{1}{K}\bar{\sigma}^{2}_{AI})}.
    \end{equation*}
\end{corollary}

\subsection{On Configuration Space with Infinite Cardinality}
Since in practice the cardinality of $\bTheta$ could be considerably
large (e.g., $10^{12}$),
in which case the bound provided
by Theorem~\ref{theorem:finite}
could be very loose.
Moreover, when the cardinality of $\bTheta$
is infinite, Theorem~\ref{theorem:finite}
does not apply anymore.
To address these issues,
we establish
new uniform error bound without dependence
on the cardinality of $\bTheta$
based on two mild assumptions given below.

\begin{assumption}\label{ass}
  \begin{enumerate}[(a)]
    \item We assume there exists $R>0$ such that $\bTheta\subseteq B_R$, where $B_{R}=\{\bw\in\rbb^h:\|\bw\|_2\leq R\}$ is a ball of radius $R$ and $\|\bw\|_2=\sum_{i=1}^{h}w_i^2$ for $\bw=(w_1,\ldots,w_h)$.
    \item We assume for any $(z,v) \in \zcal \times \vcal$, the utility function $f$ is L-Lipschitz continuous, i.e., $|f_{\theta}(z,v)-f_{\tilde{\theta}}(z,v)| \leq L||\theta -\tilde\theta||_{2}$ for all $\theta,\tilde{\theta}\in\bTheta$.
  \end{enumerate}  
\end{assumption}

Part (a) of Assumption \ref{ass} means
the ranges of the values of all parameters considered are bounded,
which holds in
nearly all practical algorithm configuration scenarios \cite{HutterLFLHLS14}.
Part (b) of Assumption \ref{ass}
poses limitations on how fast
$f_{\theta}$ can change
across $\bTheta$.
This assumption is also mild in
the sense that it is expected
that configurations with similar
parameter values would result in
similar behaviors of $\acal$,
thus getting similar performances.
The key technique for deriving the new
bound is \textit{covering numbers} as defined in Definition~\ref{def:covering_numers}.

\begin{definition}
\label{def:covering_numers}
Let $\fcal$ be a set and $d$ be a metric.
For any $\eta>0$,
a set $\fcal^\triangle \subset \fcal$
is called an $\eta$-cover of $\fcal$
if for every $f \in \fcal$ there
exists an element $g \in \fcal^\triangle$
satisfying $d(f,g) \leq \eta$.
The covering number $\ncal(\eta,\fcal, d)$
is the cardinality of the minimal
$\eta$-cover of $\fcal$:
\[
    \ncal(\eta,\fcal,d):=\min\{|\fcal^\triangle|:\fcal^\triangle\text{ is an $\epsilon$-cover of }\fcal\}.
  \]
\end{definition}

Lemma~\ref{lem:bound_on_cover_B} presents a covering number bound on $B_R$.

\begin{lemma}[\cite{pisier1999volume}]
\label{lem:bound_on_cover_B}
  \[\ln\ncal(\eta,B_R,d_2)\leq h\ln(3R/\eta),\]
  where $d_2(\bw,\tilde{\bw})=\|\bw-\tilde{\bw}\|_2$.
\end{lemma}
Since $\bTheta \subset B_{R}$,
it is easy to verify that
$\ln\ncal(\eta,\bTheta,d_2) \leq \ln\ncal(\eta,B_R,d_2)$.
Based on the L-Lipschitz continuity assumption,
Lemma~\ref{lem:bound_on_cover_F}
establishes a bound
for $\ncal(\eta,\fcal,d_\infty)$,
where $\fcal=\{f_{\theta}:\theta \in \bTheta\}$.

\begin{lemma}
\label{lem:bound_on_cover_F}
Let $\fcal=\{f_{\theta}:\theta \in \bTheta\}$ and $ d_\infty(f_{\theta},f_{\tilde{\theta}})
=\sup_{(z,v) \in \zcal \times \vcal}
|f_{\theta}(z,v)-f_{\tilde{\theta}}(z,v)|$.
If Assumption \ref{ass} holds,
then
$\ln\ncal(\eta,\fcal,d_\infty) \leq h\ln(3RL/\eta)$.
\end{lemma}
\begin{proof}
  For any $\theta,\tilde{\theta}\in \bTheta$, by the Lipschitz continuity we know $d_\infty(f_\theta,f_{\tilde{\theta}})\leq L\|\theta-\tilde{\theta}\|_2$. Then, any $(\epsilon/L)$-cover of $B_R$ w.r.t. $d_2$ would imply an $\epsilon$-cover of $\fcal$ w.r.t. $d_\infty$. This together with Lemma \ref{lem:bound_on_cover_B} implies the stated result. The proof is complete.
\end{proof}

\begin{table*}[tbp]
  \centering
  \caption{Summary of the configuration scenarios and gathered performance matrix in each scenario. h is the \#parameters of the target algorithm. Tmax is the cutoff time. The $portgen$ generator \cite{johnson2007experimental} was used to generate the TSP instances (in which the cities are randomly distributed). For each scenario, $\bTheta_{M}$ was composed of the default parameter configuration and $M-1$ random configurations. 
  }
  \scalebox{0.66}{
  \begin{tabular}{c|c|c|c|c|c|c}
  \hline
  Scenario& Algorithm& \multicolumn{1}{l|}{Domain} &Bechmark & $M$  & $P$ & \multicolumn{1}{l}{Tmax} \\ \hline
  SATenstein-QCP          & SATenstein \cite{DKhudaBukhshXHL16}, h = 54 & SAT& Randomly selected from QCP \cite{aaai/GomesS97}  & 500 & 500 & 5s \\ \hline
  clasp-weighted-sequence & clasp \cite{lpnmr/GebserKNS07a}, h=98      & ASP&  "small" type weighted-sequence \cite{padl/LierlerSTW12}  & 500 & 120 & 25s \\ \hline
  LKH-uniform-400 & LKH \cite{helsgaun2000effective}, h=23       & TSP& Generated by $portgen$ \cite{johnson2007experimental}, \#city=400  & 500 & 250 & 10s          \\ \hline
  LKH-uniform-1000 & LKH \cite{helsgaun2000effective}, h=23        & TSP& Generated by $portgen$ \cite{johnson2007experimental}, \#city=1000  & 500 & 250 & 10s \\ \hline
  \end{tabular}}
  \label{tab:scenarios}
  \end{table*}

  \captionsetup[sub]{font=small}
  \begin{figure*}[tbp]
    \centering
      \begin{subfigure}[b]{0.505\columnwidth}
        \scalebox{1.0}{\includegraphics[width=\linewidth]{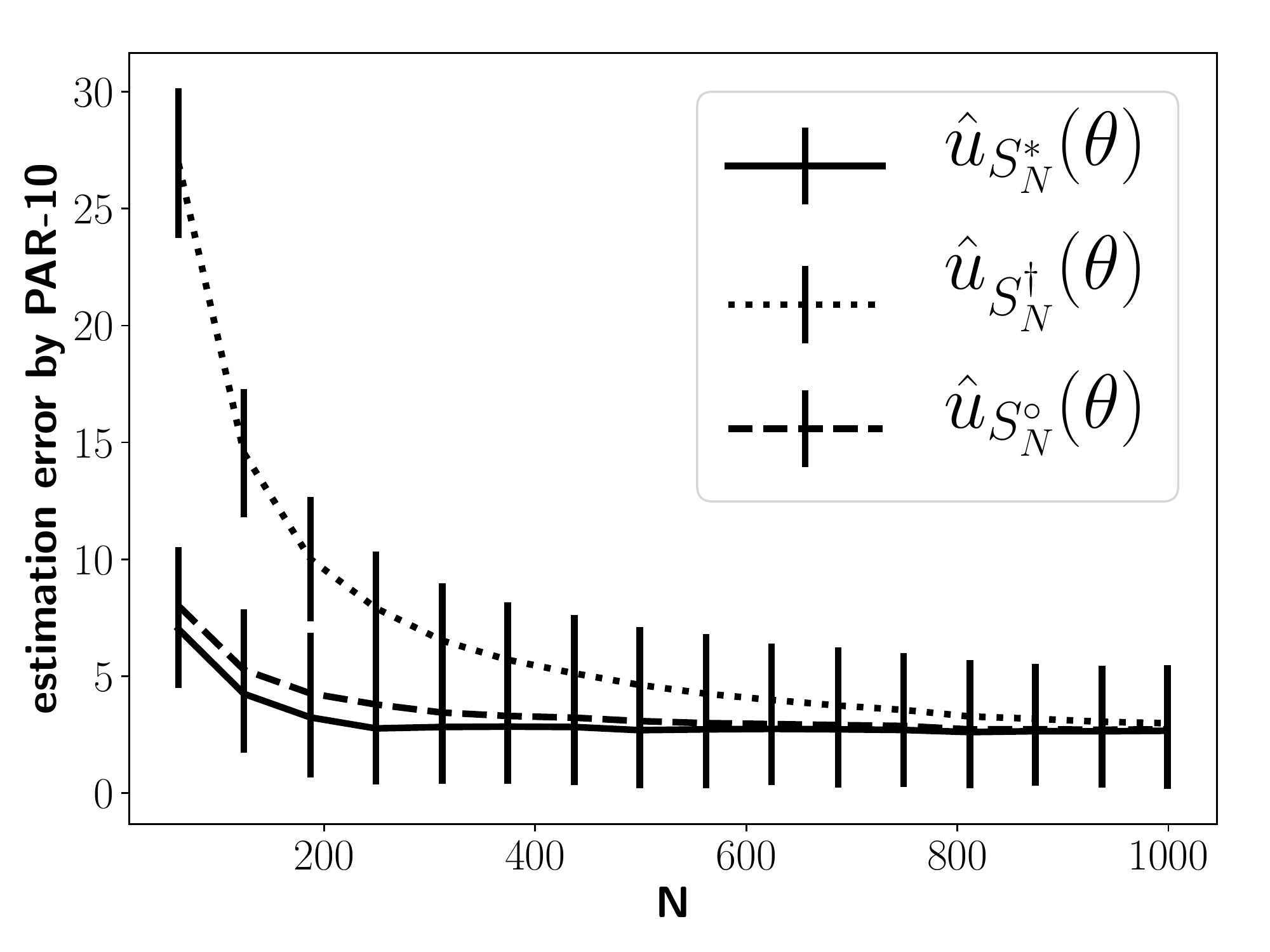}}
        \caption{SATenstein-QCP}
      \end{subfigure}
      \hfill
      \begin{subfigure}[b]{0.505\columnwidth}
        \scalebox{1.0}{\includegraphics[width=\linewidth]{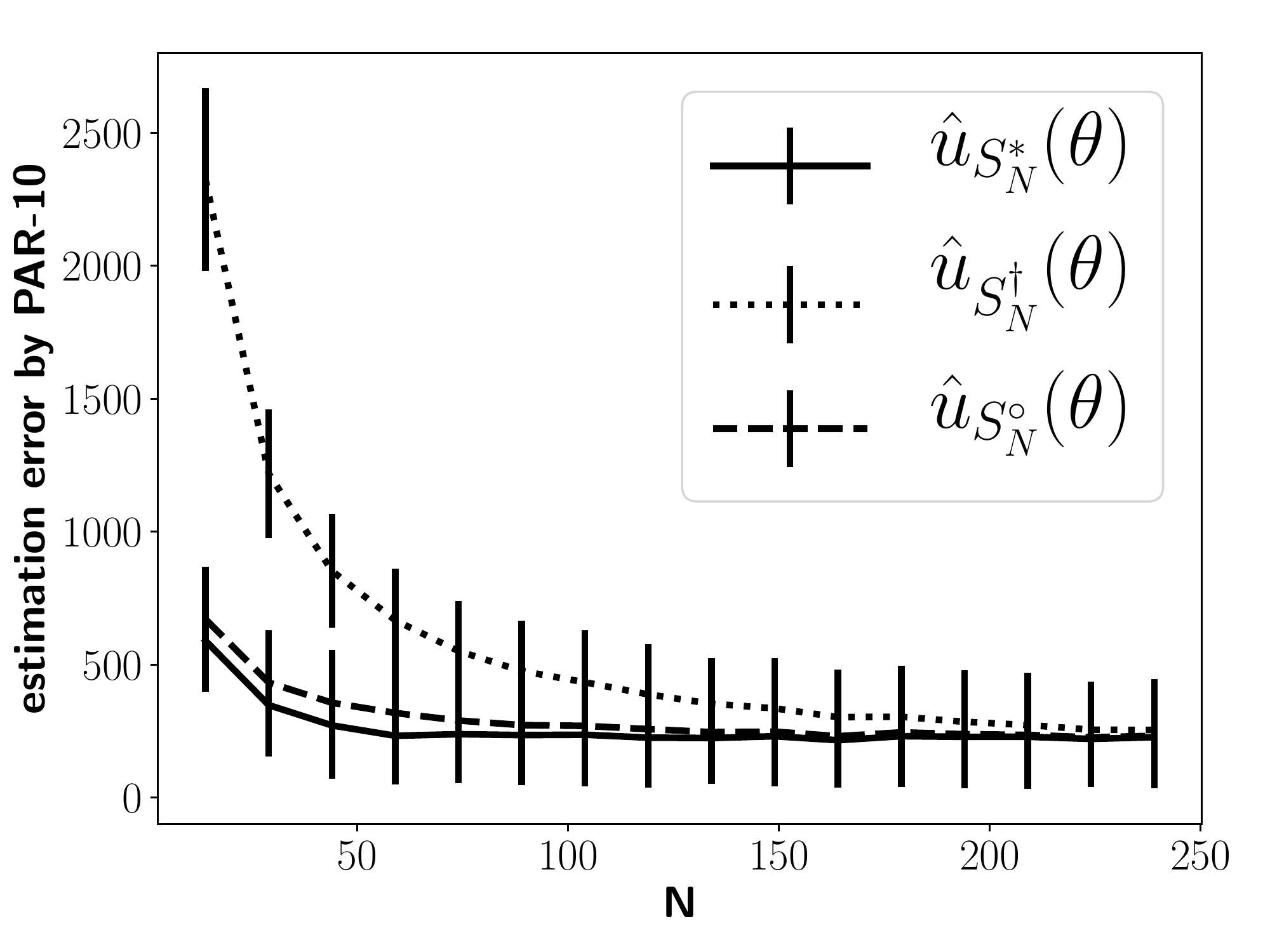}}
        \caption{clasp-weighted-sequence}
      \end{subfigure}
      \hfill
      \begin{subfigure}[b]{0.505\columnwidth}
        \scalebox{1.0}{\includegraphics[width=\linewidth]{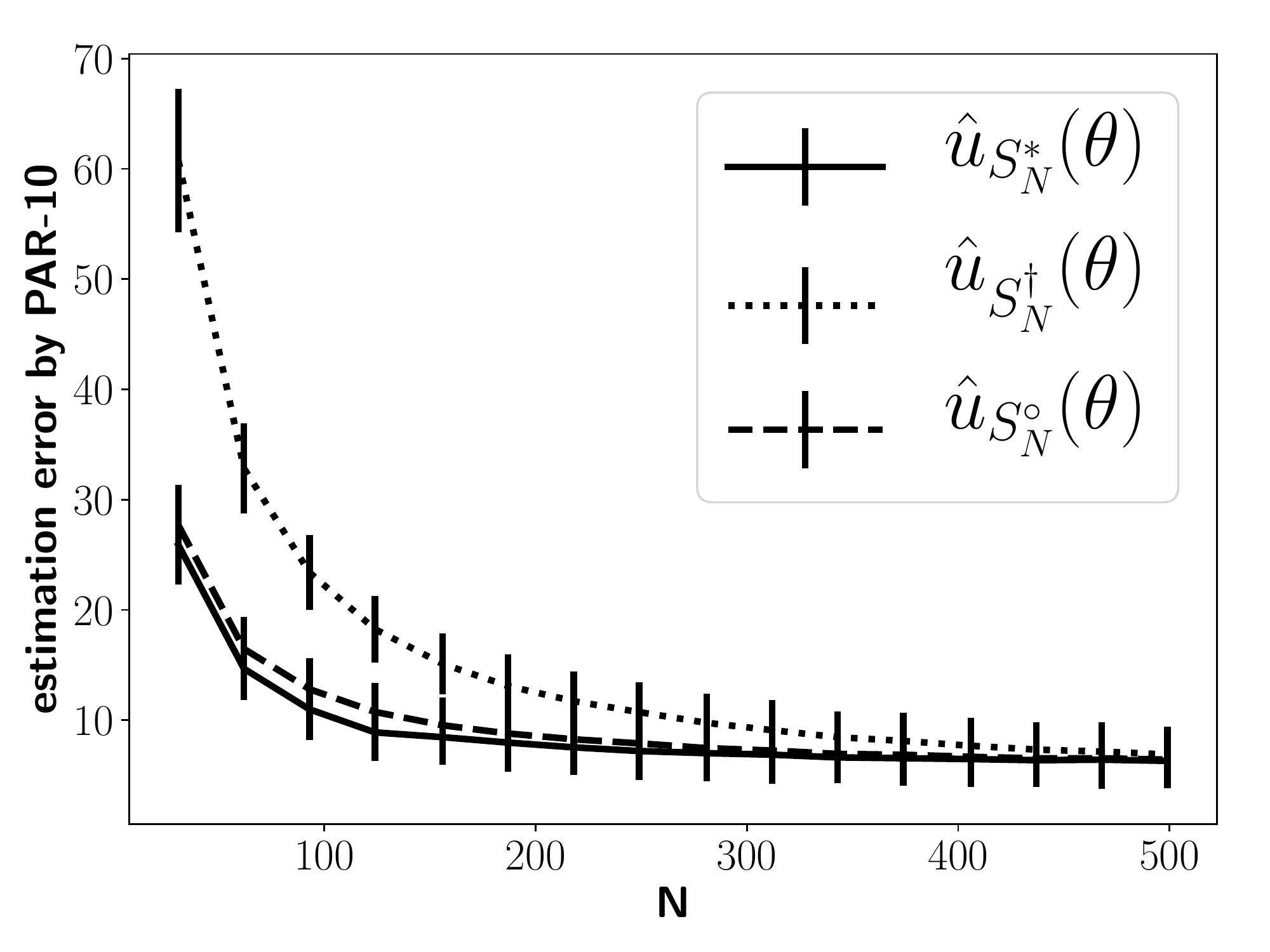}}
        \caption{LKH-uniform-400}
      \end{subfigure}
      \hfill
      \begin{subfigure}[b]{0.505\columnwidth}
        \scalebox{1.0}{\includegraphics[width=\linewidth]{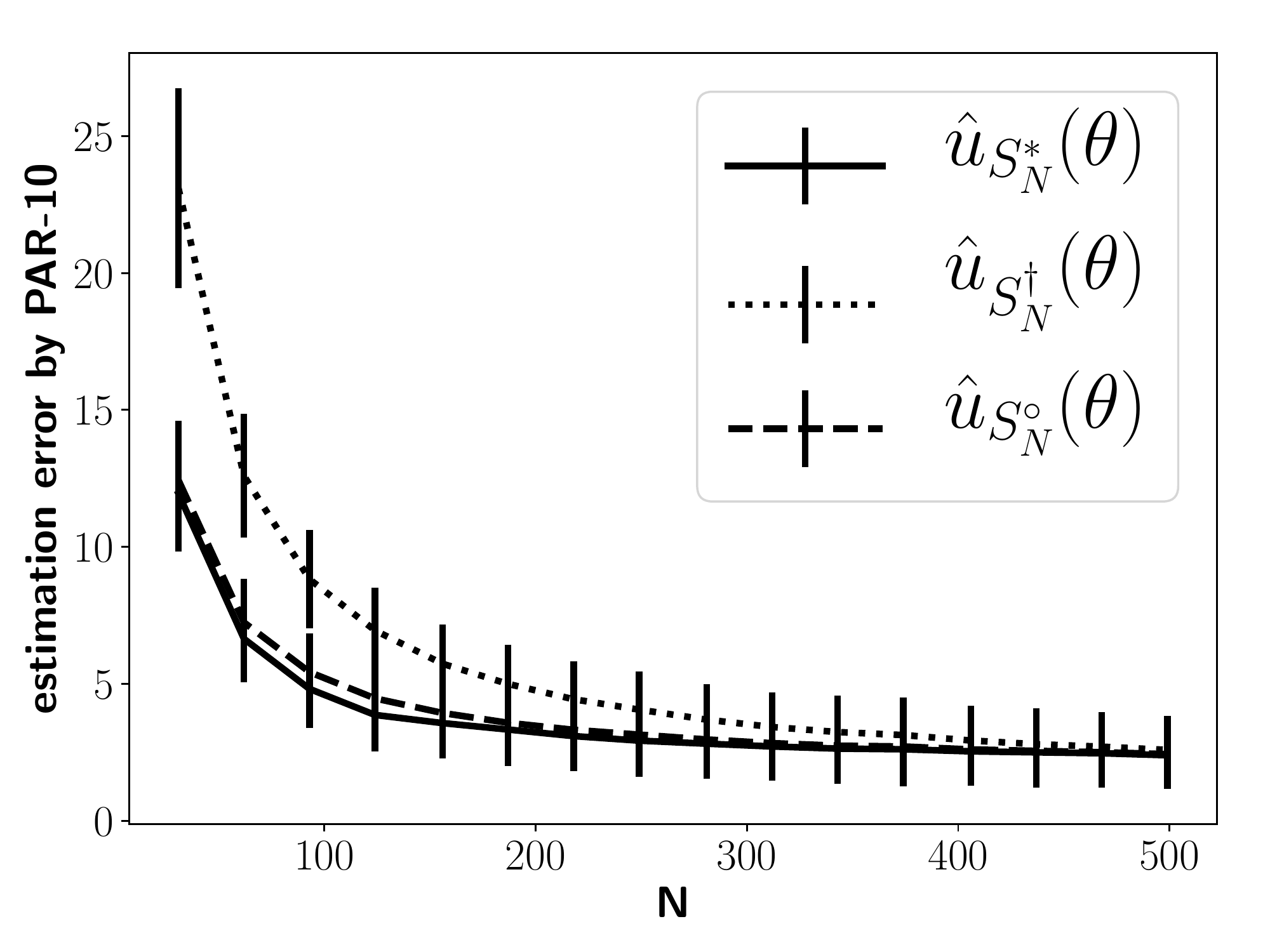}}
        \caption{LKH-uniform-1000}
      \end{subfigure}
    \caption{Estimation error for different estimators in different scenarios at $r_1=0.5$.}
    \label{fig:com_estimator}
  \end{figure*}

With the bound for $\ncal(\eta,\fcal,d_\infty)$,
the new bound for $u(\theta)-\hat{u}_{S_{N}}(\theta)$
is established in
Theorem~\ref{theorem:infinite}.

\begin{lemma}\label{lem:inequality-solution}
  For any positive constants $k, l, b, c$, the inequality $\epsilon^k l+b\ln\epsilon\geq c$ has a solution
  \begin{equation*}\label{eq:inequality-solution}
  \epsilon_0=\left(\frac{c+b\max(\ln l-\ln c,0)/k}{l}\right)^{1/k}.
\end{equation*}
\end{lemma}

\begin{theorem}
  \label{theorem:infinite}
  If Assumption \ref{ass} holds and $h\ln(12LR)\geq1$,
then for any $0<\delta<1$, with probability $1-\delta$
   there holds:
   \small
   \begin{align*}
    &\sup_{\theta \in \bTheta}[u(\theta)-\hat{u}_{S_{N}}(\theta)] \\
    & \leq
    \sqrt{\frac{h\ln(12LR)+\ln(\frac{1}{\delta})+\frac{1}{2}h\ln\frac{N}{8\tau^2+\frac{4nC}{3}}}{N}\Big(8\tau^2+\frac{4nC}{3}\Big)},
   \end{align*}
   where $n, C,\tau^{2}, \bar{\sigma}^{2}_{WI},\bar{\sigma}^{2}_{AI}$ are defined
   the same as in Theorem~\ref{theorem:finite}.
\end{theorem}
\begin{proof}
  Without loss of generality we can assume $\epsilon\leq1$.
  Let $\{f_{\theta_1},...,f_{\theta_m}\}$
  be a $\epsilon/4$-cover of $\fcal$
  with $m=\ncal(\epsilon/4,\fcal,d_\infty)$, where
  $\fcal=\{f_{\theta}:\theta \in \bTheta\}$.
 By Definition~\ref{def:covering_numers},
 for any $f_{\theta} \in \fcal$
 there exists $f_{\theta_{j}} \in
 \{f_{\theta_1},...,f_{\theta_m}\}$,
 such that
 $d_\infty(f_{\theta},
f_{\theta_{j}})
=\sup_{(z,v) \in \zcal \times \vcal}
|f_{\theta}(z,v)-f_{\theta_{j}}(z,v)|
\leq \epsilon/4$;
 it follows that
 $|\ebb[f_{\theta}(z,v)]-\ebb[f_{\theta_{j}}(z,v)]|
 =|u(\theta)-u(\theta_{j})|
 \leq \epsilon/4
 $
 and
 $
 |\hat{u}_{S_{N}}(\theta)-\hat{u}_{S_{N}}(\theta_{j})|
= \frac{1}{N}\sum_{i=1}^{K}\sum_{j=1}^{n_{i}}|f_{\theta}(z_{i}, v_{i, j})
 -
 f_{\theta_{j}}(z_{i}, v_{i, j})|
 \leq \epsilon/4.
 $
 Then,
 \small
 \begin{align*}
   & \sup_{\theta\in\bTheta}\big[u(\theta)-\hat{u}_{S_N}(\theta)\big] \leq\\
   & \sup_{\theta\in\bTheta}\Big[u(\theta)-u(\theta_j)+u(\theta_j)-\hat{u}_{S_N}(\theta_j)+\hat{u}_{S_N}(\theta_j)-\hat{u}_{S_N}(\theta)\Big]\\
   &\leq \frac{\epsilon}{2}+\max_{j\in\{1,..,m\}}[u(\theta_{j})-\hat{u}_{S_{N}}(\theta_{j})].
 \end{align*}
 It then follows that $\mathrm{Prob}\{
\sup_{\theta \in \bTheta}[u(\theta)-\hat{u}_{S_{N}}(\theta)]
\geq \epsilon
\}
\leq
\mathrm{Prob}\{
\max_{j\in\{1,..,m\}}[u(\theta_{j})-\hat{u}_{S_{N}}(\theta_{j})] \geq \epsilon/2
\}\leq \sum_{j=1}^{m}
\mathrm{Prob}\{
[u(\theta_{j})-\hat{u}_{S_{N}}(\theta_{j})] \geq \epsilon/2
\}\leq m \mathrm{exp}(-\frac{\frac{N}{4}\epsilon^{2}}{2\tau^{2}+\frac{nC\epsilon}{3}})$, where the last inequality is due to Lemma \ref{lem:concen}.
We need to find a $\epsilon$ satisfying $\mathrm{exp}(h\ln(12RL/\epsilon)-\frac{N\epsilon^{2}}{8\tau^{2}+\frac{4nC\epsilon}{3}})\leq \delta$, for which it suffices to find a solution of (by $\epsilon\leq1$)
$\frac{N\epsilon^2}{8\tau^2+\frac{4nC}{3}}+h\ln\epsilon\geq h\ln(12LR)+\ln(1/\delta)$. This inequality takes the form of the inequality in Lemma~\ref{lem:inequality-solution} (the proof of Lemma~\ref{lem:inequality-solution} is omitted here due to space limitations).
We can apply Lemma \ref{lem:inequality-solution} to show that a solution is (note $h\ln(12LR)\geq1$)
\[
\epsilon=\bigg(\frac{h\ln(12LR)+\ln(1/\delta)+2^{-1}h\ln\frac{N}{8\tau^2+\frac{4nC}{3}}}{\frac{N}{8\tau^2+\frac{4nC}{3}}}\bigg)^{\frac{1}{2}}.
\]

\end{proof}

\subsection{Discussion}
There are some important findings from the above results.
First, both 
Theorem~\ref{theorem:finite} and Theorem~\ref{theorem:infinite}
relate the bounds on $u(\theta)-\hat{u}_{S_{N}}(\theta)$ with the complexity
of $\bTheta$,
and the bounds deteriorate as the complexity increases.
This means as the considered configuration space
gets more complex, there is a possibility that the estimation
error could be larger.
Second, as expected, as
$N$ and $K$ get larger,
the estimation error gets smaller,
and $\hat{u}_{S_{N}}(\theta)$
will converge to $u(\theta)$ with
probability 1 with $N \rightarrow \infty$
and $K \rightarrow \infty$.
Third, Corollary~\ref{cor:finite}
shows that, for the estimator
$\hat{u}_{S_{N}}(\theta^{\ast})$
which are widely used in current AAC methods,
the gain on error reduction
decreases rapidly as $N$ and $K$
get larger (which are also shown in Figure~\ref{fig:analysis}
in the experiments),
and the effects of increasing $N$ and $K$
also depend on
$\bar{\sigma}^{2}_{WI}$ and
$\bar{\sigma}^{2}_{AI}$,
two quantities
varying across different
algorithm configuration scenarios.
Thus for enhancing current AAC methods, 
instead of fixing $N$ as a large number (e.g., SMAC sets $N$ to 2000 by default)
and using as many training instances as possible,
it is more desirable to use different
$N$ and $K$ according to the
configuration scenario considered,
in which case
$N$ and $K$ may be adjusted dynamically
in the configuration process as more data
are gathered to estimate
$\bar{\sigma}^{2}_{WI}$ and
$\bar{\sigma}^{2}_{AI}$.

\section{Experiments}
\label{section5}

In this section, we present our experimental studies.
First we introduce our experiment setup.
Then, we verify our theoretical results in two facets: 1) comparison of different performance estimators;
2) the effects of different values of $m$ (the number of considered configurations),
$N$ (the number of runs of $\theta$ to estimate $u(\theta)$) and
$K$ (the number of training instances)
on the estimation error.

We conducted experiments based on a re-sampling approach \cite{Birattari2004},
which is often used for time-consuming empirical analysis.
Specifically, we
considered 4 different scenarios.
We selected two scenarios SATenstein-QCP and clasp-weighted-sequence from
the Algorithm Configuration Library (AClib) \cite{HutterLFLHLS14}
and built two new scenarios LKH-uniform-400/1000.
For each scenario, we
gathered a $M \times P \times 5$ matrix containing the performances of
$M$ configurations on $P$ instances,
with each configuration running on each
instance for 5 times.
Let $\bTheta_{M}$ be the set of the $M$ configurations
and
$\zcal_{P}$ be the set of the $P$ instances.
In the experiments,
when acquiring the performance of 
a configuration $\theta$ on an instance,
instead of actually running $\theta$,
the value stored in the corresponding
entry of the matrix was used.
The details of the scenarios and the performance matrices are summarized in Table~\ref{tab:scenarios}.
In the experiments the optimization goal considered is
the runtime needed to solve the problem instances
(for SAT and ASP) or to find the optima of the problem instances (for TSP).
In particular, the performance metric was set to Penalized Average Runtime–10 (PAR-10) \cite{hutter2009paramils},
which counts a timeout as 10 times the given cutoff time.

For convenience henceforth
we will use ``$\mathrm{split}_{P_{1}|P_{2}}$''
to denote that we
subsequently randomly select,
without replacement,
$P_{1}$ and $P_{2}$
instances from $\zcal_{P}$
as training instances
and test instances respectively.
For a given $\theta$,
we always used
the performance 
obtained by an estimator
on the training instances
as its training performance,
and used its performance
on the test instances as its
true performance.
We use 
$\mathrm{uniform\_es\_error(\bTheta)}$
to denote the maximal estimation error
across the configurations in $\bTheta$.

All the experiments were conducted on a Xeon machines with 128 GB RAM and 24 cores each 
(2.20 GHz, 30 MB Cache), running CentOS.
The code was implemented based on AClib \cite{HutterLFLHLS14}
\footnote{The code and the complete experiment results are
available at \url{https://github.com/EEAAC/ac_estimation_error}}.

\captionsetup[sub]{font=small}
\begin{figure*}[tbp]
  \centering
    \begin{subfigure}[b]{0.505\columnwidth}
      \scalebox{0.9}{\includegraphics[width=\linewidth]{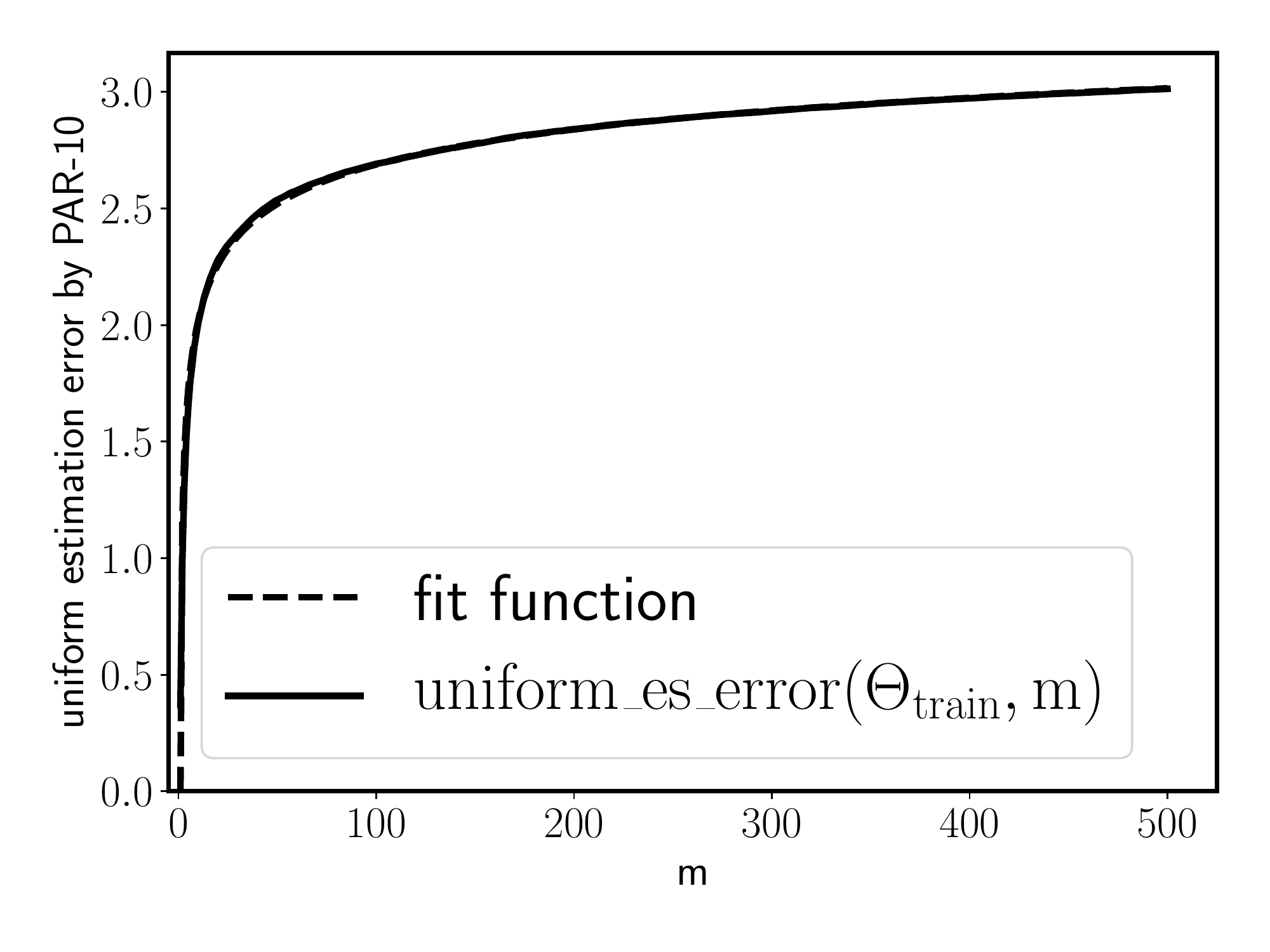}}
      \caption{SATenstein-QCP}
    \end{subfigure}
    \hfill
    \begin{subfigure}[b]{0.505\columnwidth}
      \scalebox{0.9}{\includegraphics[width=\linewidth]{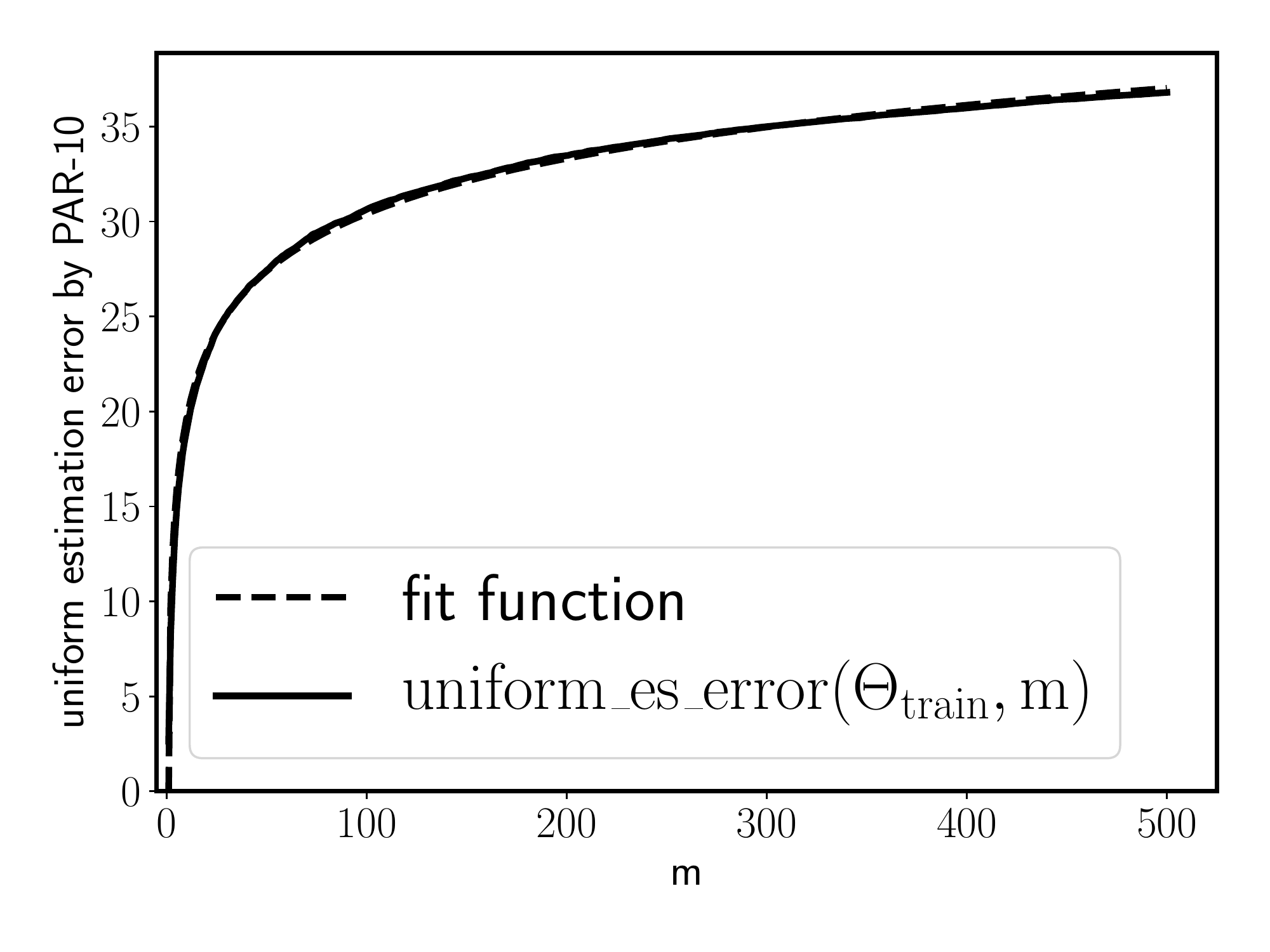}}
      \caption{clasp-weighted-sequence}
    \end{subfigure}
    \hfill
    \begin{subfigure}[b]{0.505\columnwidth}
      \scalebox{0.9}{\includegraphics[width=\linewidth]{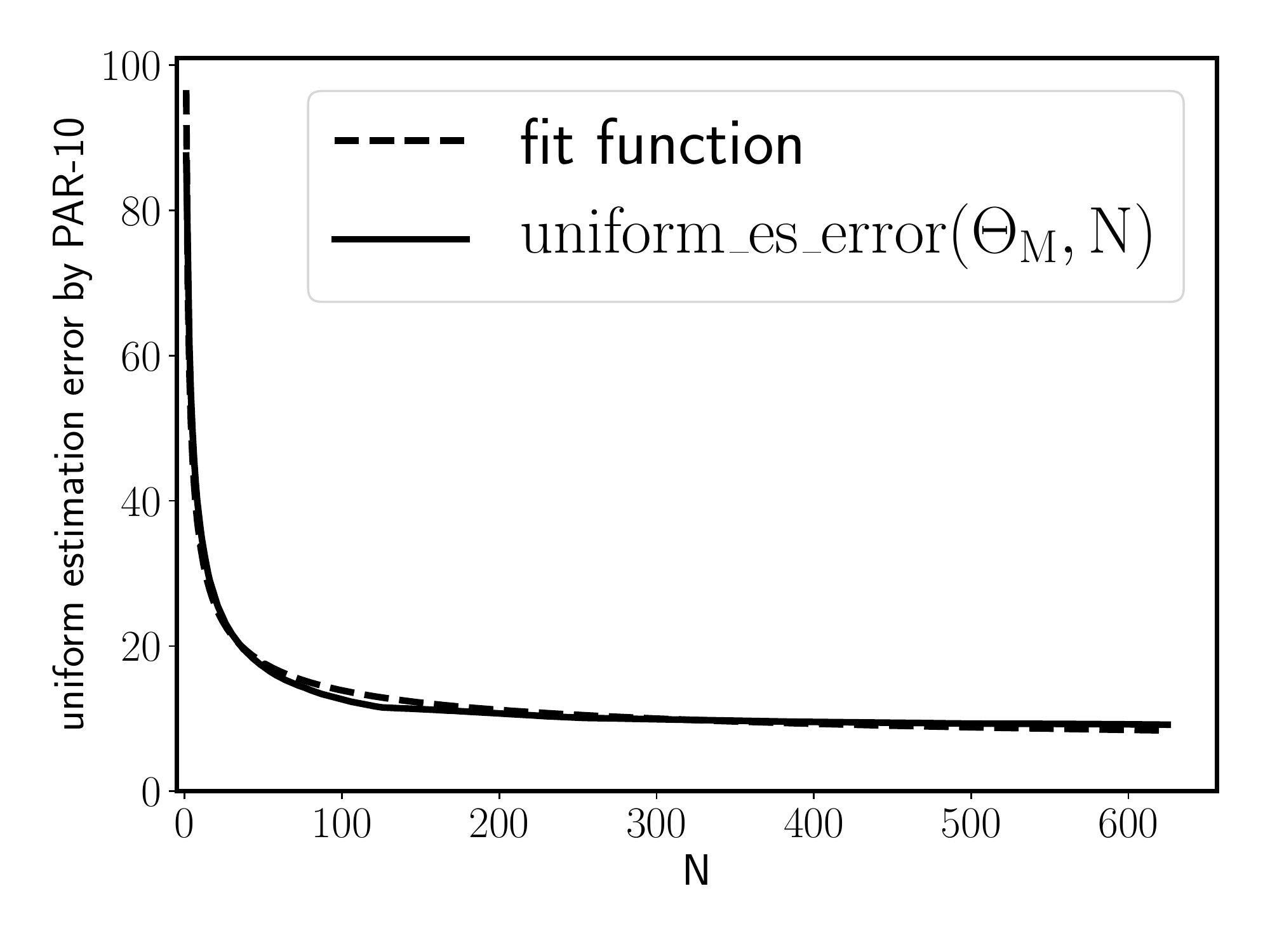}}
      \caption{LKH-uniform-400}
    \end{subfigure}
    \hfill
    \begin{subfigure}[b]{0.505\columnwidth}
      \scalebox{0.9}{\includegraphics[width=\linewidth]{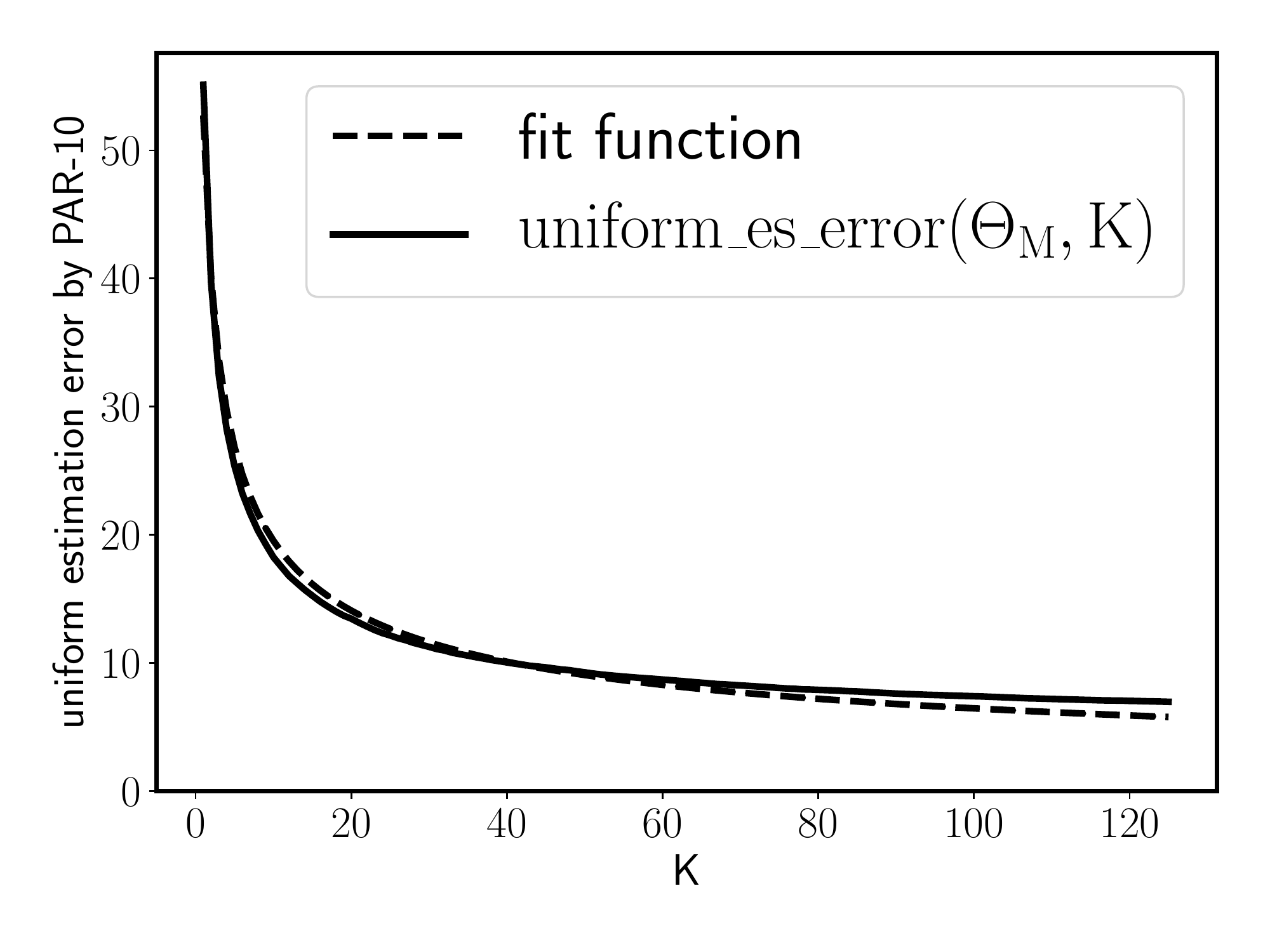}}
      \caption{LKH-uniform-1000}
    \end{subfigure}
  \caption{Uniform estimation error at different $m$, $N$ and $K$ and the fit functions based on the theoretical results.}
  \label{fig:analysis} 
\end{figure*}

\captionsetup[sub]{font=small}
\begin{figure*}[tbp]
  \centering
    \begin{subfigure}[b]{0.505\columnwidth}
      \scalebox{0.9}{\includegraphics[width=\linewidth]{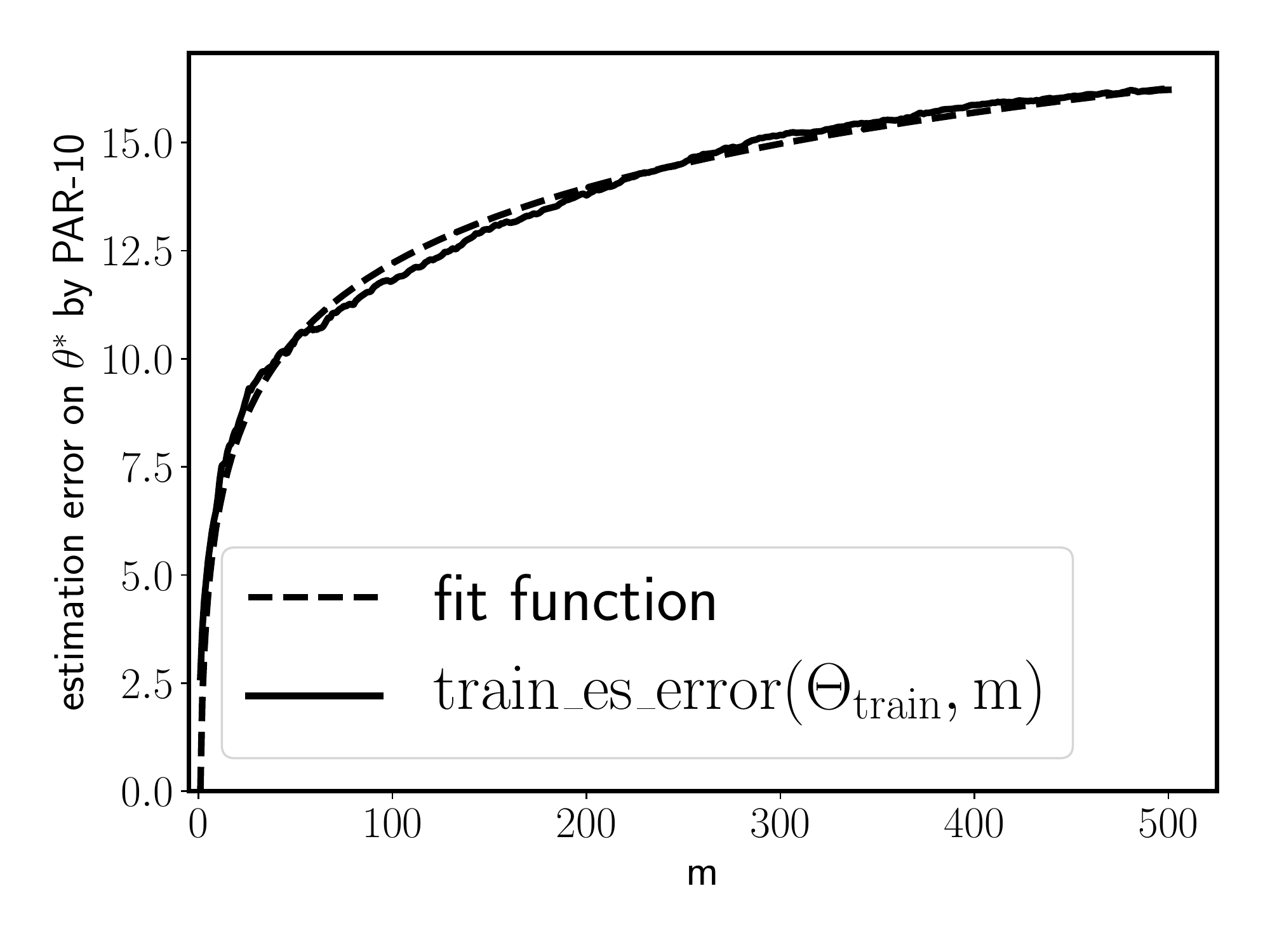}}
      \caption{clasp-weighted-sequence}
    \end{subfigure}
    \hfill
    \begin{subfigure}[b]{0.505\columnwidth}
      \scalebox{0.9}{\includegraphics[width=\linewidth]{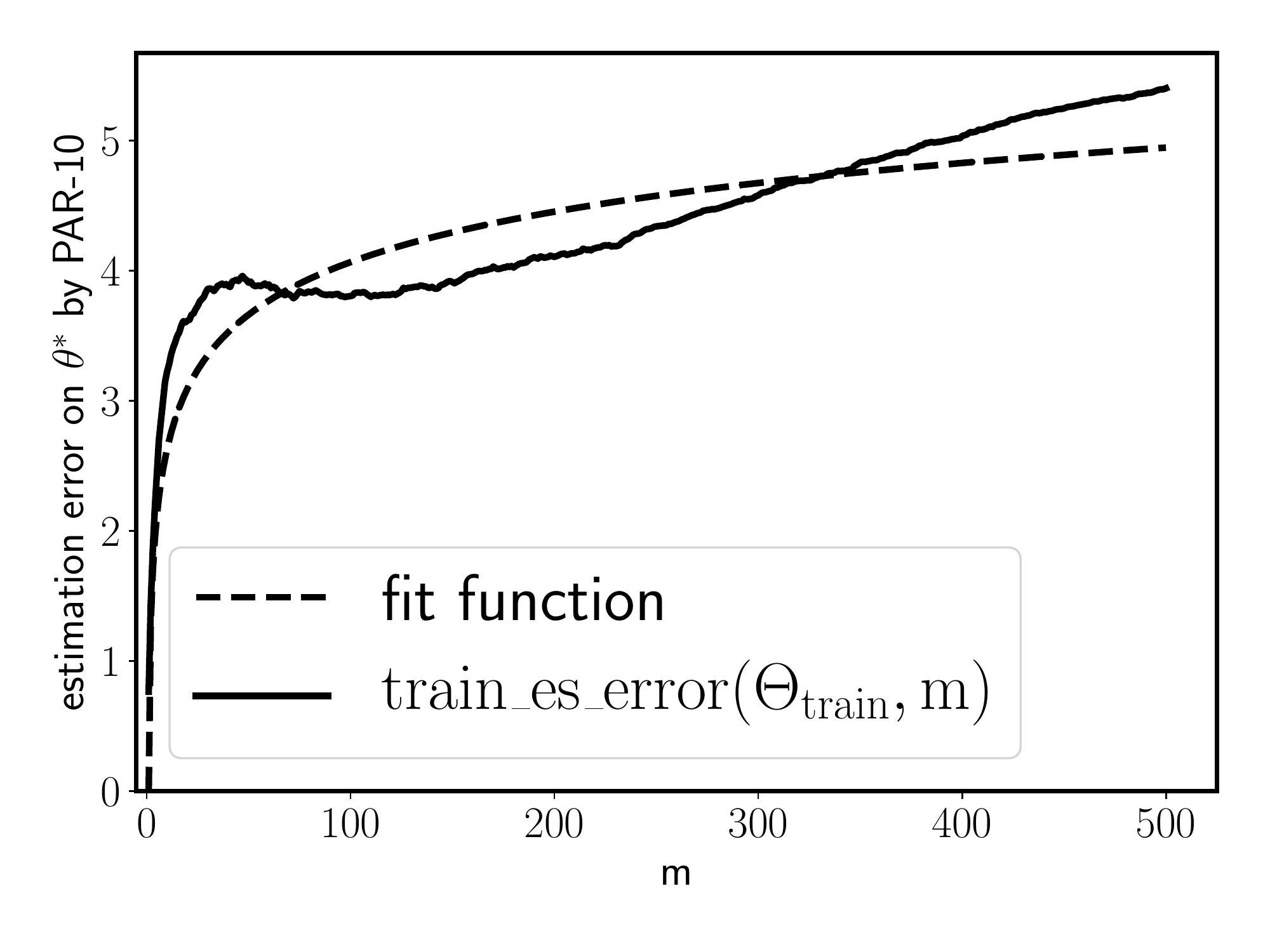}}
      \caption{LKH-uniform-1000}
    \end{subfigure}
    \hfill
    \begin{subfigure}[b]{0.505\columnwidth}
      \scalebox{0.9}{\includegraphics[width=\linewidth]{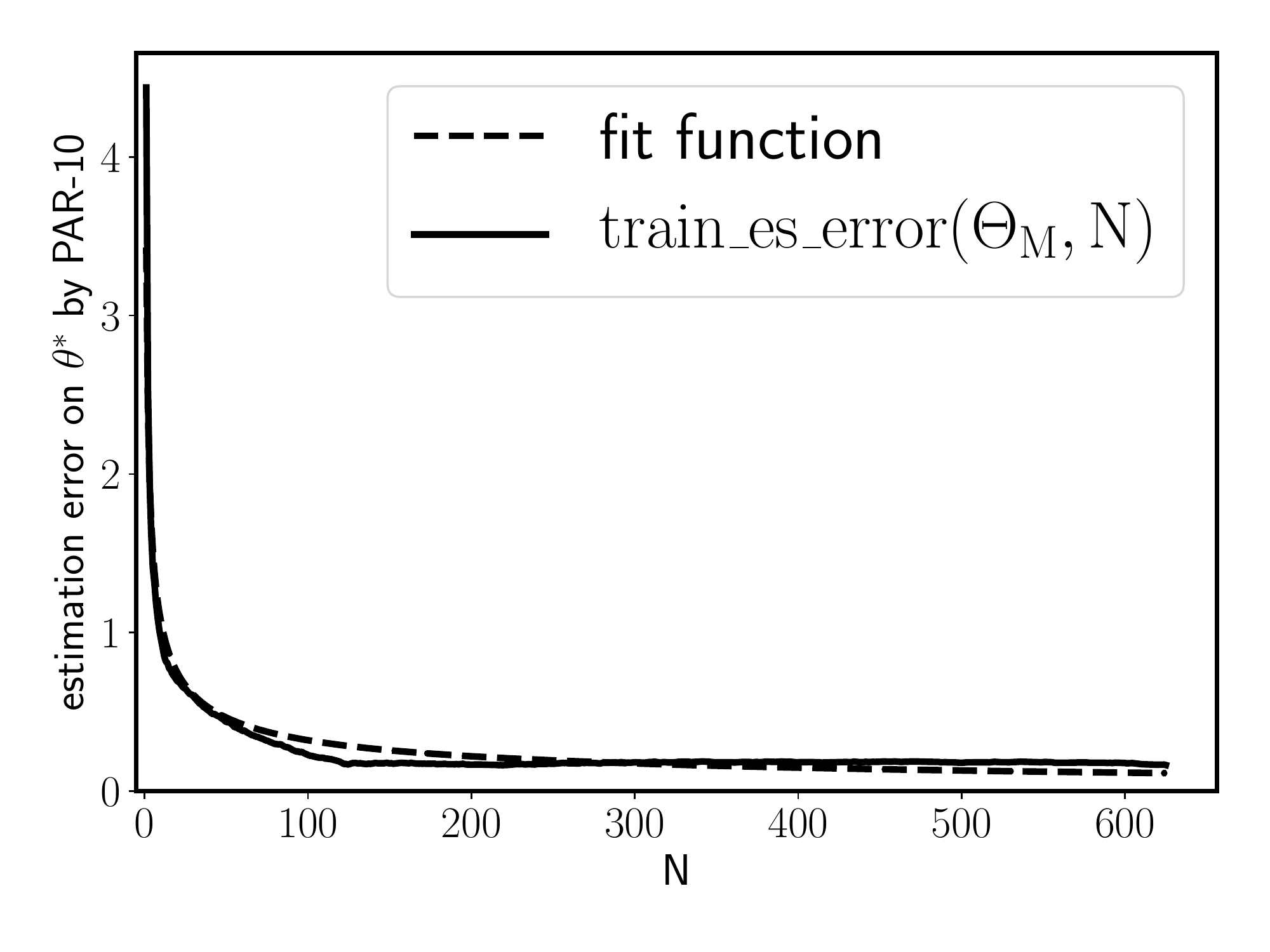}}
      \caption{LKH-uniform-400}
    \end{subfigure}
    \hfill
    \begin{subfigure}[b]{0.505\columnwidth}
      \scalebox{0.9}{\includegraphics[width=\linewidth]{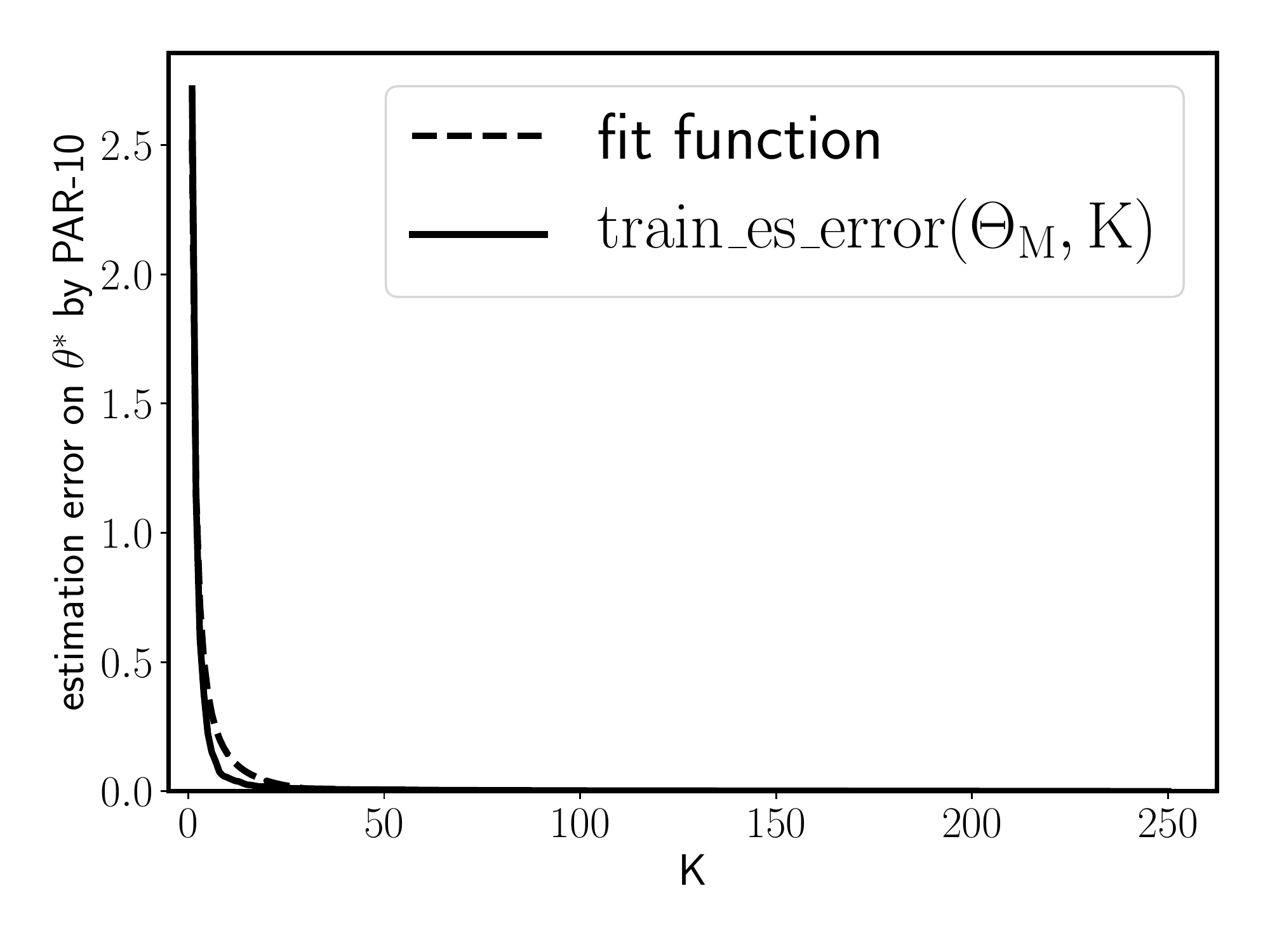}}
      \caption{SATenstein-QCP}
    \end{subfigure}
    \caption{Estimation error on
    $\theta^{\ast}$ at different $m$, $N$ and $K$ and the fit functions based on the theoretical results.}
    \label{fig:train_bound}
\end{figure*}

\textbf{Comparison of Different Estimators.}
We compared $\hat{u}_{S^{\ast}_{N}}(\theta)$ with two
estimators $\hat{u}_{S^{\dagger}_{N}}(\theta)$ and $\hat{u}_{S^{\circ}_{N}}(\theta)$.
For evaluating $\theta$,
$\hat{u}_{S^{\dagger}_{N}}(\theta)$
repeatedly randomly selects an instance
from $\zcal_{P}$ without replacement,
and runs $\theta$ for 5 times on the instance as long as the total number
of runs of $\theta$ not exceeding $N$.
$\hat{u}_{S^{\dagger}_{N}}(\theta)$ is greedier 
than $\hat{u}_{S^{\ast}_{N}}(\theta)$ in the sense that it ensures the estimated performance of $\theta$ on the used instances is as accurate as possible.
Another estimator $\hat{u}_{S^{\circ}_{N}}(\theta)$
is the one presented in \cite{Birattari2004},
which repeatedly randomly selects an instance from $\zcal_{P}$ with replacement,
and runs $\theta$ for a single time on the instance.
$\hat{u}_{S^{\circ}_{N}}(\theta)$ has more
randomness than $\hat{u}_{S^{\ast}_{N}}(\theta)$
since it does not ensure that
$N$ runs of $\theta$ are distributed evenly on all instances.
We set $K=r_{1}P$ and $N=r_{2}K$, and
ranged $r_{1}$ from 0.1 to 0.5 with a step of 0.05,
$r_{2}$ from 0.25 to 4.0 with a 
step of 0.25.
To reduce the variations of our experiments,
for each combination of $r_{1}$ and $r_{2}$,
we $\mathrm{split}_{{K}|P/2}$
for 2500 times,
and on each split,
we obtained the estimation error
of an estimator
on each $\theta \in \bTheta$,
and then calculated
the mean value,
which was further averaged
over all splits.
That is, for each combination
of $r_{1}$ and $r_{2}$,
we obtained a mean estimation
error for each estimator.
Due to space limitations, we only present
the results in terms of error bars (mean $\pm$ std)
at $r_{1}=0.5$ in Figure~\ref{fig:com_estimator}.
The results at other values of $r_{1}$
are similar.
Figure~\ref{fig:com_estimator} is
in line with Theorem~\ref{theorem:best_estimator}.
Overall $\hat{u}_{S^{\ast}_{N}}(\theta)$
is the best estimator among the three,
and its performance advantage is remarkable
when $N$ is small. When $N$ gets larger,
it is expected, and as shown in Figure~\ref{fig:com_estimator},
that the estimation error for
all three estimators will
converge to 0.
The fact that $\hat{u}_{S^{\ast}_{N}}(\theta)$
is better than $\hat{u}_{S^{\circ}_{N}}(\theta)$
indicates that it is necessary to
distribute $N$ runs of $\theta$
as evenly as possible over all instances.

\textbf{Estimation Error at Different $\bm{m}$, $\bm{N}$ and $\bm{K}$.}
We always fixed two values while ranging the other one.
We ranged $m$ from 1 to $M$,
while setting $K=P/2$ and $N=5K$.
We ranged $N$ from 1 to $5K$
while setting $K=P/2$ and $m=M$.
We ranged $K$ from 1 to $P/2$
while setting $N=5K$ and $m=M$.
For a given $m$,
we $\mathrm{split}_{K|P/2}$
for 2500 times,
and on each split,
we started with an empty set
$\bTheta_{train}$ of configurations and 
then repeatedly
expanded $\bTheta_{train}$
by adding a configuration
randomly selected from
$\bTheta_{M}\setminus\bTheta_{train}$.
Each time a new configuration $\theta$
was added to $\bTheta_{train}$,
$\mathrm{uniform\_es\_error(\bTheta_{train})}$
was recorded,
which was further averaged
over all 2500 splits.
That is, for each $m$,
we obtained a mean value of
$\mathrm{uniform\_es\_error(\bTheta_{train})}$,
denoted as $\mathrm{uniform\_es\_error(\bTheta_{train}, m)}$.
Similarly, for a given $N$
or a given $K$,
we always $\mathrm{split}_{K|P/2}$
for 2500 times,
and on each split,
we obtained
$\mathrm{uniform\_es\_error(\bTheta_{M})}$,
and then averaged it over all splits.
Thus for each considered
$N$ and $K$,
we obtained 
$\mathrm{uniform\_es\_error(\bTheta_{M}, N)}$
and
$\mathrm{uniform\_es\_error(\bTheta_{M}, K)}$, respectively.
Due to space limitations,
we only present parts of the results
in Figure~\ref{fig:analysis}
and other results 
are very similar.
To verify whether our analysis
(Theorem~\ref{theorem:finite})
correctly captures the dependence of estimation error on
$N$, $M$ and $K$,
we also plot the function
$f(m)=a\ln{m}+b\sqrt{\ln{m}}$
for $m$,
$f(N)=a+b\sqrt{1/N}$ for $N$
and 
$f(K)=a/K+b\sqrt{1/K}$
for $K$
in Figure~\ref{fig:analysis},
where the parameters
$a,b$ are computed by
fitting $f$ with
the data collected in the experiments,
i.e.,
$\{m \mapsto \mathrm{uniform\_es\_error(\bTheta_{train}, m)}:m \in \{1,...,M\}\}$,
$\{N \mapsto \mathrm{uniform\_es\_error(\bTheta_{M}, N)}:N \in \{1,...,\frac{5}{2}P\}\}$
and
$\{K \mapsto \mathrm{uniform\_es\_error(\bTheta_{M}, K)}: K \in \{1,...,\frac{1}{2}P\}\}$, respectively.
Overall Figure~\ref{fig:analysis}
demonstrates that our analysis managed to 
capture the dependence of uniform estimation error on $m$, $N$ and $K$.
It is worth noting that in the experiments
the effects of increasing $m$, $N$ and $K$
depend on $\bar{\sigma}^{2}_{WI}$ and
$\bar{\sigma}^{2}_{AI}$,
which vary across
configuration scenarios.
Moreover,
the estimation error becomes very low
and quite stable when $N$ approaches $K/2$,
which means running $\theta$ on 
half of the training instances
could already obtain a reliable estimate
of $u(\theta)$.

It is also meaningful to investigate whether
our analysis could reflect how 
the estimation error
on $\theta^{\ast}$,
i.e., the configuration with
the best training performance in $\bTheta$,
denoted as $\mathrm{train\_es\_error(\bTheta)}$,
would change.
We conducted the same experiments as described above to gather $\mathrm{train\_es\_error(\bTheta_{train}, m)}$,
$\mathrm{train\_es\_error(\bTheta_{M},N)}$
and
$\mathrm{train\_es\_error(\bTheta_{M},K)}$.
Figure~\ref{fig:train_bound} plots the results and the fit functions.
It could be seen that although
the bounds are not directly established
for $\mathrm{train\_es\_error(\bTheta)}$,
the findings also apply to
it to a considerable extent.

\section{Conclusion}
The main results of this paper include
the universal best performance estimator and
bounds on the uniform estimation error,
which were verified in extensive experiments.
Possible future directions include
data-dependent bounds that are tighter and
computable from realization of training instances
and analysis of the notorious over-tuning phenomenon
based on the results in this paper.

\bibliographystyle{aaai}
\fontsize{9.0pt}{10.3pt} \selectfont
\bibliography{mybib}

\end{document}